\documentclass{article}

\usepackage[nonatbib,preprint]{neurips_2026}

\usepackage[utf8]{inputenc} 
\usepackage[T1]{fontenc}    
\usepackage{hyperref}       
\usepackage{url}            
\usepackage{booktabs}       
\usepackage{amsfonts}       
\usepackage{amsmath, amssymb, amsthm, mathtools}
\usepackage{enumitem}
\usepackage{nicefrac}       
\usepackage{microtype}      
\usepackage{xcolor}         
\usepackage{bm}
\usepackage{bbm}

\usepackage[dvipdfmx]{graphicx}
\usepackage[dvipdfmx]{xcolor}
\usepackage{tikz}
\usetikzlibrary{calc}

\setlist{leftmargin=1.5em}

\theoremstyle{definition}
\newtheorem{definition}{Definition}
\newtheorem{theorem}{Theorem}
\newtheorem{proposition}{Proposition}
\newtheorem{lemma}{Lemma}

\newtheorem{remark}{Remark}
\newtheorem{example}{Example}

\newcommand{\Emph}[1]{\textbf{#1}}
\newcommand{\NewTerm}[1]{\textit{\textbf{#1}}}

\newcommand\dhxrightarrow[2][]{%
  \mathrel{\ooalign{$\xrightarrow[#1\mkern4mu]{#2\mkern4mu}$\cr%
  \hidewidth$\rightarrow\mkern4mu$}}
}

\newcommand{\Enc}{\operatorname{\mathsf{Enc}}}

\newcommand{\NN}{\mathbb{N}}
\newcommand{\QQ}{\mathbb{Q}}
\newcommand{\RR}{\mathbb{R}}
\newcommand{\EE}{\operatorname{\mathbb{E}}}
\newcommand{\dom}{\operatorname{\mathsf{dom}}}
\newcommand{\src}{\operatorname{\mathsf{src}}}
\newcommand{\computableto}{\xrightarrow{\mathsf{comp}}}
\newcommand{\computableinjection}{\dhxrightarrow{\mathsf{comp}}}
\newcommand{\supp}{\operatorname{supp}}

\newcommand{\Interpret}{\operatorname{\mathsf{Interpret}}}

\newcommand{\Programs}{\operatorname{\mathsf{Programs}}}

\title{Fundamental Limitation in Explaining AI}

\author{%
  Atsushi Suzuki \\
  Department of Mathematics\\
  Faculty of Science\\
  The University of Hong Kong \\
  Hong Kong SAR\\
  \texttt{atsushi.suzuki.rd@outlook.com} \\
  \And
  Jing Wang \\
  School of Computing and Mathematical Sciences \\
  Faculty of Engineering and Science \\
  University of Greenwich \\
  United Kingdom \\  \texttt{jing.wang.research@gmail.com} \\
}

\begin{document}

\maketitle

\begin{abstract}
While large-scale models such as LLMs and diffusion models have achieved practical success, public institutions have emphasized the importance of explainability in AI. Existing methods for explaining AI, however, are not designed to provide completely faithful explanations of the behavior of large-scale AI systems.
Although a completely faithful and interpretable explanation of the behavior of an AI system might be useful for AI governance, it has not been known whether providing such an explanation is theoretically possible.
In this paper, we mathematically prove a fundamental quadrilemma in explaining AI, stating that AI and its explanation cannot satisfy the following four conditions simultaneously: 1) the complexity of the operation environment, 2) the goodness of the AI's performance, 3) the interpretability of the AI's explanation, and 4) the complete faithfulness of the AI's explanation.
This quadrilemma suggests that, in most applications where we cannot change the environment or sacrifice good AI performance and an interpretable explanation, we should give up complete faithfulness of explanations and should instead aim to explain only the parts that are important for applications. As a consequence, the quadrilemma implies that AI governance should be designed on the premise that the faithfulness of AI explanations is always incomplete.
\end{abstract}

\section{Introduction}

With the major successes of GPT-2 \cite{radford2019language} and latent diffusion models \cite{rombach2022high} as milestones, large-scale AI systems such as large language models (LLMs) and diffusion models have achieved success on various tasks, including question-answering and text-to-image generation. These AI systems define input-output relations that return desirable outputs to user inputs, either deterministically or stochastically.
As large-scale AI systems achieve industrial success even on complex tasks, public institutions have emphasized the importance of explainability in AI (e.g., NIST, US \cite{nist2024genai_profile}, European Data Protection Supervisor, EU \cite{edps2025airiskmanagement}).
However, existing methods for explaining AI are not designed to provide a completely faithful explanation of the behavior of a practical AI system so that the explanation can describe the whole behavior of the AI system and distinguish it from any other AI systems.
For example, some provide numerical sequences corresponding to only a part of the properties of practical AI systems used for complex tasks such as general-purpose chat AI, for example local properties or scalarized outputs (e.g., SHAP \cite{lundberg2017unified} and its derivative studies \cite{lundberg2020local}, integrated gradients \cite{sundararajan2017axiomatic}, integrated Hessians \cite{janizek2021explaining}, NormLIME \cite{ahern2019normlime}, Grad-CAM \cite{selvaraju2017grad} and its derivative methods \cite{chattopadhay2018grad,wang2020score,fu2020axiom}, counterfactual explanations \cite{wachter2017counterfactual}, and algorithmic recourse \cite{karimi2021algorithmic}), some generate text describing a part of such properties (e.g., \cite{huang2023can, kroeger2023context, jie2024interpretable}), others construct surrogate models whose performance is reduced for the sake of explanation (e.g., LIME \cite{ribeiro2016should}, Anchors \cite{ribeiro2018anchors}, GLocalX \cite{setzu2021glocalx}, GLEAMS \cite{visani2024gleams}, and decision-tree-based methods \cite{craven1995extracting, setiono1995understanding, jacobs2022ai,buono2024expected}).
The lack of faithfulness in AI explanations has often been recognized as a current problem \cite{turpin2023language, zhao2024explainability}.
If a practical AI system could be equipped with an interpretable explanation that is completely faithful to its behavior, AI governance might become easier. However, whether providing such an explanation is possible in principle has not been clarified.
Therefore, this paper asks the following research question: \Emph{under a complex environment}, does there exist a pair consisting of an \Emph{AI system with sufficiently good performance for practical use} and its explanation that is \Emph{completely faithful} to its behavior and \Emph{interpretable} to humans?

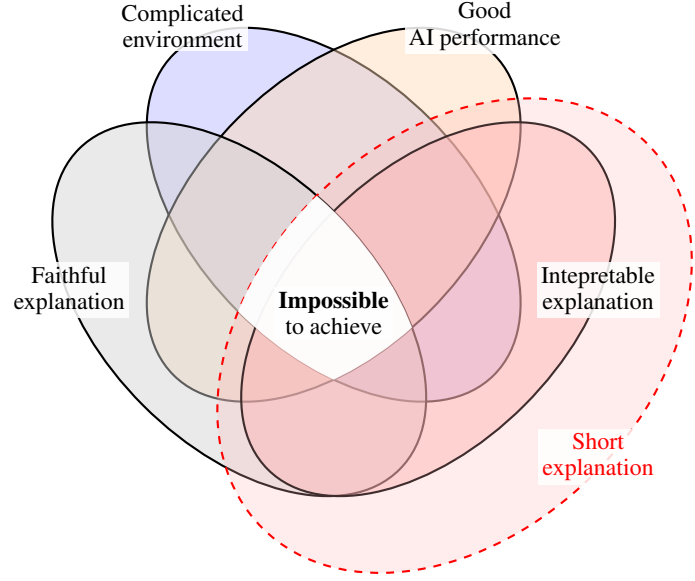
\begin{figure}[t]
\centering
\begin{tikzpicture}[
  every node/.style={
    font=\small,
    align=center,
    fill=white,
    fill opacity=0.85,
    text opacity=1,
    inner sep=1.5pt
  }
]

\filldraw[
  fill=blue!30,
  draw=black,
  line width=0.8pt,
  fill opacity=0.45,
  draw opacity=1,
  rotate around={-45:(0,1)}
]
  (0,1) ellipse [x radius=3, y radius=1.8];

\filldraw[
  fill=orange!30,
  draw=black,
  line width=0.8pt,
  fill opacity=0.45,
  draw opacity=1,
  rotate around={45:(0,1)}
]
  (0,1) ellipse [x radius=3, y radius=1.8];

\filldraw[
  fill=gray!35,
  draw=black,
  line width=0.8pt,
  fill opacity=0.45,
  draw opacity=1,
  rotate around={-45:(-1.25,-0.25)}
]
  (-1.25,-0.25) ellipse [x radius=3, y radius=1.8];

\filldraw[
  fill=red!30,
  draw=black,
  line width=0.8pt,
  fill opacity=0.45,
  draw opacity=1,
  rotate around={45:(1.25,-0.25)}
]
  (1.25,-0.25) ellipse [x radius=3, y radius=1.8];

\filldraw[
  fill=red!30,
  dashed,
  draw=red,
  line width=0.8pt,
  fill opacity=0.25,
  draw opacity=1,
  rotate around={45:(1.6,-0.6)}
]
  (1.6,-0.6) ellipse [x radius=3.6, y radius=2.6];

\begin{scope}
  \clip[rotate around={-45:(0,1)}]
    (0,1) ellipse [x radius=3, y radius=1.8];
  \clip[rotate around={45:(0,1)}]
    (0,1) ellipse [x radius=3, y radius=1.8];
  \clip[rotate around={-45:(-1.25,-0.25)}]
    (-1.25,-0.25) ellipse [x radius=3, y radius=1.8];
  \clip[rotate around={45:(1.6,-0.6)}]
    (1.6,-0.6) ellipse [x radius=3.6, y radius=2.6];

  \fill[white, fill opacity=0.95] (-6,-6) rectangle (6,6);
\end{scope}

\node at (-2, 3.5) {Complicated\\environment};
\node at ( 2, 3.5) {Good\\AI performance};
\node at (-3.5,0) {Faithful\\explanation};
\node at ( 3.5,0) {Intepretable\\explanation};
\node[text=red] at ( 3.5,-2.2) {Short\\explanation};
\node at (0, -0.3) {\textbf{Impossible}\\to achieve};
\end{tikzpicture}
\label{fig:Concept}
\caption{Conceptual diagram illustrating the main claim of this study. In this paper, we prove that the following four conditions cannot all be satisfied simultaneously: the complexity of the operation environment, the goodness of AI performance, the interpretability of AI explanations, and the complete faithfulness of AI explanations. Among these, it is difficult to give a mathematical formulation of the interpretability of AI explanations directly. However, the shortness of an AI explanation is a necessary condition for interpretability and is much easier to formulate mathematically. Therefore, in this paper, under the other three conditions, we prove the impossibility of achieving short AI explanations, and as a consequence, we show the impossibility of achieving interpretable AI explanations.}
\end{figure}

This paper answers the above research question negatively. Specifically, with our novel inequality, we mathematically show that there is a fundamental \Emph{quadrilemma} stating that an AI system and its explanation cannot simultaneously satisfy the following four conditions (See Figure \ref{fig:Concept} for the conceptual diagram).

\begin{itemize}
\item \textbf{Complexity of the operation environment}: The AI operates under an environment in which the true input-output relation is complex.
\item \textbf{Goodness of the AI's performance}: The AI's performance is sufficiently good.
\item \textbf{Interpretability of the AI's explanation}: The explanation is interpretable to humans.
\item \textbf{Complete faithfulness of the AI's explanation}: The explanation is completely faithful to the behavior of the AI.
\end{itemize}
Although this quadrilemma is logically a negative statement, it provides guidance for the field of AI explanation as to which conditions should be prioritized and which condition should be sacrificed.
More specifically, in applications involving natural language or image generation, such as recent large-scale AI systems, it is unavoidable to operate in environments where the true input-output relation is complex. Moreover, the goodness of performance and the interpretability of explanations cannot be sacrificed. Therefore, in many cases, one should give up the requirement that explanations be completely faithful to AI behavior.
This conclusion justifies the direction of existing AI explanation methods, which are not designed to provide completely faithful explanations.
The quadrilemma in this paper suggests that future research on AI explanation methods should not attempt to provide completely faithful explanations of AI behavior, but should instead be designed to explain only those parts of AI behavior that are important for applications, and that industry and public institutions should design AI governance on the premise that AI explanations will continue to be unfaithful.

One significant strength of this paper is that the quadrilemma is given mathematically as a form of an inequality without ambiguity. While the incompleteness of definitions of the relevant concepts has often been regarded as a problem in the field of AI explanation (e.g., \cite{arrieta2020explainable, kumarage2026explainable}), in this paper, in contrast, the necessary conditions for the four conditions appearing in the quadrilemma are all mathematically formulated, which enables avoiding ambiguity in our result.
Specifically, the necessary conditions for the four conditions are formulated as follows.
Firstly, the goodness of AI performance is formulated by using perplexity, which is a basic performance metric in natural language processing.
Secondly, the complete faithfulness of an AI explanation is formulated by requiring not only that the explanation uniquely specify the behavior of the AI, but also that, under an appropriate interpretation program, the concrete probability distribution of the output can be approximately computed to arbitrary precision from the explanation and the input string.
This is because, if simple injectivity from behaviors to strings were adopted as the definition of complete faithfulness, then meaningless assignments to strings would be allowed as long as there were no duplicates.
Thirdly, as a necessary condition for the interpretability of an explanation, we adopt the \Emph{shortness of the explanation}.
The shortness of the explanation is not sufficient, but necessary for the explanation to be interpretable since human capacity to read and remember text is limited \cite{brysbaert2019many} \cite{guinness_pi_places_memorised} (See also Appendix \ref{sec:HumanLimitation}).
The trick here is that, generally speaking, we can prove an impossibility, like a quadrilemma, by refuting a necessary condition only, without proving the original condition.
This trick releases us from the difficulty in formulating the interpretability of explanations directly, which has often been pointed out in the field of AI explanation \cite{arrieta2020explainable, kumarage2026explainable}. As a result, we can substitute the shortness, which is easy to handle mathematically, for the interpretability in proving the quadrilemma.  
Fourthly, we use \Emph{Kolmogorov complexity}, a classical information-theoretic measure, to quantify the complexity of the operation environment.
The nontrivial issue is which object's Kolmogorov complexity to use.
Interestingly, this study proves that na\"ively choosing the Kolmogorov complexity $C (P (\cdot \mid \cdot))$ of the conditional probability mass function $P (\cdot \mid \cdot)$ of the operation environment as a measure of the complexity of the operation environment fails to describe the limitations on AI performance and explanation shortness.
To avoid this problem, this paper adopts the true expectation of conditional Kolmogorov complexity, $\EE_{X, Y \sim P} C (Y \mid X)$, as the measure of the complexity of the operation environment.

Under the above formulations of the four conditions, this paper proves an inequality stating that the sum of the expected logarithmic conditional perplexity and the length of a completely faithful explanation is, up to terms of logarithmic order or smaller, cannot be smaller than the true expectation of conditional Kolmogorov complexity.
This means that, under a complex environment in which the true expectation of conditional Kolmogorov complexity is very large, either the expected logarithmic conditional perplexity becomes large, meaning AI's poor performance, or the length of a completely faithful explanation becomes long, meaning losing interpretability of the explanation; thus, the quadrilemma arises.
This inequality suggests that, unless one sacrifices AI performance, one must give up the complete faithfulness of explanations.

The contributions of this paper are as follows.

\begin{itemize}
    \item We prove a novel inequality showing a fundamental quadrilemma in explaining AI, stating that AI and its explanation cannot satisfy the following four conditions simultaneously: 1) the complexity of the operation environment, 2) the goodness of the AI's performance, 3) the interpretability of the AI's explanation, and 4) the complete faithfulness of the AI's explanation.
    \item We apply the derived inequality to practical applications and show that giving up complete faithfulness of AI explanations is unavoidable. We further justify, from the derived inequality, that explaining only part of the behavior of an AI system or using a surrogate model as an explanation can serve as a solution.
    \item We demonstrate that the Kolmogorov complexity of the conditional probability mass function of the environment cannot describe the limitation on the goodness of AI performance and the interpretability of completely faithful explanations, and that instead we need to use the expectation of conditional Kolmogorov complexity as a measure of the complexity of the operating environment.
\end{itemize}

\section{Related work}
\label{sec:Related}
There are studies that discuss trade-offs concerning AI explanations, but to the best of the authors' knowledge, there is no paper that discusses a trade-off among the complexity of the environment, AI performance, faithfulness of explanations, and interpretability.
Zhang et al.~\cite{zhang2023trade} show, within the scope of removal-based explanations such as SHAP \cite{lundberg2017unified}, the impossibility of simultaneously satisfying interpretability, consistency, and efficiency under a situation in which an interpretable function class is given in advance. However, the class of explanations is limited to removal-based explanations, and the interpretable class must be specified separately. Moreover, they do not clarify the trade-off with the performance of the AI that is explained by such explanations.
Bilodeau et al.~\cite{bilodeau2024impossibility} also theoretically show limitations only within the scope of feature attribution methods such as SHAP \cite{lundberg2017unified}, but their result is not a limitation on general explanation methods and does not clarify the trade-off with AI performance. Moreover, these studies do not analyze explanations that are completely faithful to AI behavior.
Bressan et al.~\cite{bressan2024theory} discuss the trade-off between the height of a tree and approximation performance for a method that explains a model by a decision tree whose branching conditions are written in the form of indicator functions of subsets, and Frost et al.~\cite{frost2024partially} discuss trade-offs of explanations based on decision lists. However, neither directly handles AI systems with stochastic input-output behavior, such as recent LLMs and diffusion models. In addition, explanations by decision trees or decision lists do not necessarily take into account explanations that are simplified by being structured with subroutines.
The present study differs from previous studies in that it does not restrict the type of explanation, and shows a trade-off that takes into account all of the complexity of the environment, AI performance, complete faithfulness of explanations, and interpretability of explanations, in a setting that also includes stochastic input-output behavior.

This paper explains AI behavior from the viewpoints of Kolmogorov complexity and perplexity.
Both are information-theoretically natural quantities. For example, Proposition 6 of Shportko \cite{shportko2026kolmogorov} includes a relation between them, but it is in the context of steganography and does not directly provide implications for explainability.
The relationship between LLMs and Kolmogorov complexity has recently been suggested in many studies \cite{burden2025conversational,pan2025understanding,deletang2024language,shportko2026kolmogorov,elmoznino2025incontext}, but to the best of the authors' knowledge, there is no study that directly shows the impossibility of explanations based on this relationship.

The remainder of this paper is organized as follows.
Section \ref{sec:ProblemFormulation} formulates the problem setting considered in this paper and the four conditions appearing in the quadrilemma.
Section \ref{sec:MainResult} presents the inequality that gives the quadrilemma as the main result of this paper and explains its implications.
It also explains that a naive method fails.
Section \ref{sec:Limitation} clarifies the limitations of the present method and suggests directions for future work.

\section{Problem Formulation}
\label{sec:ProblemFormulation}

The goal of this paper is to derive a quadrilemma concerning AI explanations.
In this section, after giving a comprehensive definition of AI that includes LLMs and diffusion models, we mathematically rigorously define the four conditions appearing in the quadrilemma.
We first introduce mathematical notation.

\noindent \textbf{Notation}: The set of nonnegative integers is denoted by $\NN$, the set of rational numbers by $\QQ$, and the set of real numbers by $\RR$.
The set of natural numbers less than $n$ is denoted by $\NN_{<n}$. That is, $\NN_{<n}=\{0,1,\ldots,n-1\}$.
General sets are denoted by calligraphic letters such as $\mathcal{X}$ and $\mathcal{Y}$.
For a set $\mathcal{A}$, its cardinality is denoted by $|\mathcal{A}|$.
The set of all finite sequences consisting of elements of $\mathcal{A}$ is denoted by $\mathcal{A}^*$.
That is, $\mathcal{A}^* := \mathcal{A}^0 \cup \mathcal{A}^1 \cup \cdots$.
Finite sequences are written in bold italic letters such as $\bm{a}$, and their length is denoted by $|\bm{a}|$.
When we want to explicitly write the elements of a sequence, we write, for example, $\bm{a}=(a_0,a_1,\ldots,a_{m-1})$.
The concatenation of a string $\bm{a}$ and a string $\bm{b}$ is written as $\bm{a}\cdot\bm{b}$.
The empty sequence, that is, the sequence of length zero, is written as $()$.

\subsection{Definition of stochastic input-output AI}

We first need to define AI mathematically. To make the scope of the theory as broad as possible, we want the definition of AI to be as broad as possible.
Therefore, in this paper, we define AI as an input-output relation, defined on a computer, that is appropriate for solving the task we want to solve.
For example, in the case of a general-purpose chat AI, the input is a string obtained by arranging the history of prompts entered by the user with tags, and the output is a string desired by the user.
That is, a chat AI is an appropriate input-output relation whose input set is a set of strings and whose output set is a set of strings.
In the case of text-to-image AI, the input is a prompt string, and the output is a vector representing integer values of each pixel of an image.
That is, a text-to-image AI is an appropriate input-output relation whose input set is a set of strings and whose output set is a set of integer-valued vectors.
We also want to include not only deterministic input-output relations but also stochastic input-output relations using pseudo-random numbers, as in LLMs and diffusion models.
Here, the important point is that all functions constituting AI are \Emph{computable}.
A function being computable means, intuitively, that it is a function that can be described by a finite string called a program.
For a more detailed definition, see Appendix \ref{sec:Preliminary}.
Although the definition is cumbersome, the reasons for restricting the class of functions considered to computable functions are the following two points.
\begin{itemize}
    \item Since AI is a function implemented on a computer, it necessarily belongs to the class of computable functions, and it is sufficient to consider only that class.
    \item The set of all general functions is uncountable, whereas explanations are finite strings, and the set of finite strings over a finite alphabet is at most countable. Therefore, assigning individual explanations to all general functions is trivially impossible, leaving no further room for consideration.
\end{itemize}
Based on the above considerations, we define a \NewTerm{stochastic input-output AI} composed of computable functions as follows.
Let $\Sigma$ be a finite alphabet with at least two elements, used by the computation model when defining computability.
For the countable sets appearing below, we assume that encoding functions into $\Sigma^*$ have been fixed, each of which is injective and has a computable (decidable) image, and that computability is defined based on these encodings.
When $f$ is a partial computable function with source set $\mathcal{X}$ and target set $\mathcal{Y}$, we write $f:\subseteq \mathcal{X}\computableto\mathcal{Y}$. The domain of $f$ is written as $\dom f$. If $x \in \dom f$ then we write $f (x) \downarrow$.
\begin{definition}[Stochastic input-output AI]
\label{def:AI}
Let the input space, or source set, be $\mathcal{X}$ and the output space, or target set, be $\mathcal{Y}$, and suppose that $\mathcal{X}$ and $\mathcal{Y}$ are at most countable.
A \NewTerm{stochastic input-output AI} with domain $\mathcal{X}'\subseteq\mathcal{X}$ that uses sequences of $n$-element uniform random numbers is a pair $A=(f,\tau)$ consisting of a partial computable function $f:\subseteq \big(\NN_{<n}\big)^*\times\mathcal{X}\computableto\mathcal{Y}$, called the \NewTerm{main function} in this paper, which receives a random-number sequence and an input and returns an output, and a partial computable function $\tau:\subseteq \big(\NN_{<n}\big)^*\times\mathcal{X}\computableto\{0,1\}$, called the \NewTerm{random-sequence acceptance decision function} in this paper, satisfying the following properties.
\begin{itemize}
\item \textbf{Consistency between the main function and the domain.}
If $\bm{x}\in\mathcal{X}\setminus\mathcal{X}'$, then for every $\bm{u}\in\big(\NN_{<n}\big)^*$, $(\bm{u},\bm{x})\notin\dom f$.
\item \textbf{Definability of the accepted random-sequence set.}
$\dom\tau\supseteq\big(\NN_{<n}\big)^*\times\mathcal{X}'$.
In particular, for every $\bm{x}\in\mathcal{X}'$, the accepted random-sequence set
$\mathcal{T}_{\bm{x}} := \Big\{\bm{u}\in\big(\NN_{<n}\big)^* \ \Big|\  \tau(\bm{u},\bm{x})\downarrow = 1 \Big\}$ can be defined.
By definition, note that $\mathcal{T}_{\bm{x}}$ is a computable set.
\item \textbf{Prefix-freeness and constructive stopping guarantee of the accepted random-sequence set.}
For every $\bm{x}\in\mathcal{X}'$, the accepted random-sequence set $\mathcal{T}_{\bm{x}}\subseteq(\NN_{<n})^*$ is prefix-free.
Furthermore, there exists a computable function
$N_{\mathrm{stop}}:\mathcal{X}'\times\NN\computableto\NN$
such that, for every $\bm{x}\in\mathcal{X}'$ and every $k\in\NN$, the following holds:
\begin{equation}
1-
\sum_{\substack{\bm u\in\mathcal T_{\bm x}\\ |\bm u|\le N_{\mathrm{stop}}(\bm x,k)}}
n^{-|\bm u|}
\le
|\Sigma|^{-k}.
\end{equation}
\item \textbf{Consistency between the domain of the main function and the accepted random-sequence set.}
If $\bm{x}\in\mathcal{X}'$ and $\bm{u}\in\mathcal{T}_{\bm{x}}$, then $(\bm{u},\bm{x})\in\dom f$.
\end{itemize}
\end{definition}
The above definition of stochastic input-output AI may look cumbersome at first glance, but its naturalness can be seen by considering the intended operation.
Roughly speaking, when the input is $\bm{x}$, uniform random numbers are repeatedly sampled and concatenated, and at the point when the random-number sequence enters the accepted random-sequence set, that random-number sequence and the original input $\bm{x}$ are fed into the main function.
More formally, the intended operation of stochastic input-output AI is described as follows.
\begin{definition}[Intended operation of stochastic input-output AI]
\label{def:AIProcedure}
Let $\mathcal{X}$ be the input space and $\mathcal{Y}$ be the output space, and let $A=(f,\tau)$ be a stochastic input-output AI with domain $\mathcal{X}'\subseteq\mathcal{X}$.
When $A$ receives $\bm{x}\in\mathcal{X}'$ as input, it can produce a stochastic output according to the following procedure, using an oracle $\mathsf{Uniform}_{<n}$ that gives at each time a value regarded as a sample from the discrete uniform distribution on $\{0,1,\ldots,n-1\}$.
Here, $\mathcal{T}_{\bm{x}}$ is the accepted random-sequence set determined by $\tau$, as defined in Definition \ref{def:AI}.
\begin{itemize}
\item \textbf{Input}: A uniform-distribution oracle $\mathsf{Uniform}_{<n}$ and a string input $\bm{x}\in\mathcal{X}$.
\item \textbf{Step 1}: Initialize $\bm{u}\gets ()$ as the empty string.
\item \textbf{Step 2}: If $\bm{u}\in\mathcal{T}_{\bm{x}}$, then let $f$ accept $\bm{u}$, \textbf{output} $f(\bm{u},\bm{x})$, and \textbf{terminate} (undefined if $f(\bm{u},\bm{x})\uparrow$).
\item \textbf{Step 3}: $u_\mathsf{new}\gets\mathsf{Uniform}_{<n}$.
\item \textbf{Step 4}: $\bm{u}\gets\bm{u}\cdot u_\mathsf{new}$. Then go to \textbf{Step 2}.
\end{itemize}
The above stochastic operation is called the \NewTerm{intended operation} of the stochastic input-output AI $A$.
\end{definition}

\begin{remark}[Relation to pseudo-random number generators]
Ideally, the uniform random oracle $\mathsf{Uniform}_{<n}$ is a true random number generator that independently generates a random number following the true uniform distribution supported on $\{0,1,\ldots,n-1\}$ at each time.
In implementation, however, it is not realistic to use a true random number generator as the uniform random oracle $\mathsf{Uniform}_{<n}$, and pseudo-random numbers are used.
For example, the Mersenne Twister, a representative pseudo-random algorithm, can be regarded as a uniform random oracle with $n=2^{32}$.
\end{remark}

We now confirm that AI systems used in practice are examples of stochastic input-output AI.

\begin{example}[Stochastic input-output AI includes deterministic input-output relations, diffusion models, and LLMs]
Definition \ref{def:AI} includes deterministic input-output relations, AI models such as diffusion models that consume a fixed number of pseudo-random numbers, and AI models such as LLMs in which the number of consumed pseudo-random numbers is determined depending on the values of the pseudo-random numbers generated so far.
More specifically:
\begin{itemize}
\item \textbf{Deterministic functions}: Any deterministic input-output function represented by a partial computable function is a kind of stochastic input-output AI. Specifically, by setting $\mathcal{T}_{\bm{x}}=\{()\}$ for every $\bm{x}\in\mathcal{X}'$, one obtains a deterministic input-output relation.
\item \textbf{Diffusion models}: In the case of diffusion models, the number of consumed uniform random numbers is often fixed. For a diffusion model consuming $\ell$ uniform random numbers, by setting
\[
\mathcal{T}_{\bm{x}}=\Big\{\bm{u}\in\big(\NN_{<n}\big)^* \ \Big| \ |\bm{u}|=\ell\Big\}
\]
and defining an appropriate main function, it is formulated as a stochastic input-output AI.
\item \textbf{Large language models}: In the case of large language models, the number of consumed random numbers is not predetermined. Rather, a fixed number of random numbers is consumed each time one output token is determined, and the number of consumed random numbers differs depending on the number of output tokens. In addition, a termination condition that can be determined by computation is set in advance, such as the generation of an EOS (end of sentence) token. Whether the termination condition is satisfied is determined only by the random-number sequence $\bm{u}$ generated so far and the input $\bm{x}$. Therefore, when the input is $\bm{x}$, by setting
\[
\mathcal{T}_{\bm{x}}
=
\left\{\bm{u}\in\big(\NN_{<n}\big)^* \ \middle| \ \begin{aligned} &\text{$\bm{u},\bm{x}$ satisfy the termination condition,} \\ &\text{and for every proper prefix $\bm{u}'$ of $\bm{u}$, $\bm{u}',\bm{x}$ do not satisfy the termination condition}\end{aligned}\right\}
\]
and defining an appropriate main function, it is formulated as a stochastic input-output AI.
\end{itemize}
\end{example}

The stochastic behavior of the input and output of a stochastic input-output AI can be completely described by a conditional probability mass function.

\begin{definition}[Conditional probability mass function determined by stochastic input-output AI]
For a stochastic input-output AI $A$, we write $Q_A$ for the conditional probability mass function obtained when the procedure in Definition \ref{def:AIProcedure} is executed using a true discrete uniform random number generator as $\mathsf{Uniform}_{<n}$.
Here, $Q_A(\cdot\mid\cdot):\mathcal{X}'\times\mathcal{Y}\computableto[0,1]$, and the probability mass function of the output when $\bm{x}\in\mathcal{X}'$ is used as input is denoted by $Q_A(\cdot\mid\bm{x}):\mathcal{Y}\computableto[0,1]$.

Moreover, noting that the probability that any $\bm{u}\in\mathcal{T}_{\bm{x}}\subsetneq\big(\NN_{<n}\big)^*$ is accepted is $n^{-|\bm{u}|}$, the concrete value of $Q_A(\cdot\mid\cdot)$ is given as follows:
\begin{equation}
Q_A(\bm{y}\mid\bm{x})
=
\sum_{\bm{u}\in\mathcal{T}_{\bm{x}}} n^{-|\bm{u}|}\mathbbm{1}\big(f(\bm{u},\bm{x})=\bm{y}\big).
\end{equation}
\end{definition}

\begin{remark}[Meaning of satisfying Kraft's inequality with equality]
Once $\mathcal{T}_{\bm{x}}$ is prefix-free, Kraft's inequality
\(
\sum_{\bm{u}\in\mathcal{T}_{\bm{x}}}n^{-|\bm{u}|}\le 1
\)
is necessarily satisfied.
This inequality being satisfied with equality, namely
\(
\sum_{\bm{u}\in\mathcal{T}_{\bm{x}}}n^{-|\bm{u}|}=1,
\)
is a necessary and sufficient condition for the algorithm defined by Definition \ref{def:AI} to halt with probability $1$ and for $Q_A(\cdot\mid\bm{x})$ to be a probability mass function rather than a genuine sub-probability.
\end{remark}

Now, as stated above, the behavior of a stochastic input-output AI $A$ is described by the conditional probability mass function $Q_A(\cdot\mid\cdot)$ that it determines.
Thus, formally speaking, explaining a stochastic input-output AI means explaining $Q_A(\cdot\mid\cdot)$.
Therefore, it is important to know what function class $Q_A(\cdot\mid\cdot)$ belongs to.
In fact, the conditional probability mass function determined by a stochastic input-output AI $A$ is always computable.

\begin{proposition}[The conditional probability mass function determined by stochastic input-output AI is computable]
\label{prp:AIandComputable}
For every stochastic input-output AI $A$, $Q_A(\cdot\mid\cdot)$ is always a computable function.
\end{proposition}

For the above reasons, in what follows, it is sufficient for this paper to consider only explanations of computable conditional probability mass functions.

\begin{remark}[The above framework is not limited to black-box explanation]
Since $Q_A(\cdot\mid\cdot)$ is the stochastic behavior of the input-output relation, explaining it may at first appear to be possible only by treating the AI as a black box.
However, if the numerical values of internal states are directly concatenated to the output $\bm{y}$ and defined as a new output, the following discussion can be interpreted, without modification, as a discussion of limitations in the case where internal states as well as the input-output relation are explained.
\end{remark}

\subsection{Definition of the four conditions of the quadrilemma}

\noindent \textbf{Formulation of AI performance based on perplexity}: As an evaluation criterion for stochastic input-output AI, perplexity is widely used, especially in the field of natural language processing (e.g., \cite{bengio2003neural,mikolov2010recurrent,radford2019language}).
This paper also adopts the logarithm of conditional perplexity and its expectation as evaluation criteria.

\begin{definition}[Perplexity]
\label{def:Perplexity}
Let the input space be $\mathcal{X}$ and the output space be $\mathcal{Y}$.
Given an input $\bm{x}\in\mathcal{X}$ and an output $\bm{y}\in\mathcal{Y}$, the \NewTerm{perplexity} of a conditional probability mass function $Q(\cdot\mid\cdot)$ is defined as
\(
\frac{1}{Q(\bm{y}\mid\bm{x})}.
\)
When a true distribution $P$ on $\mathcal{X}\times\mathcal{Y}$ is given, the expected logarithmic perplexity is defined as
$\EE_{\bm{X},\bm{Y}\sim P} -\log Q(\bm{Y}\mid\bm{X})$.
\end{definition}

\begin{remark}[On the definition of perplexity]
In cases such as large language models where the input space $\mathcal{X}$ can be written as $\mathcal{X}=\Lambda^*$ using an alphabet $\Lambda$, the quantity
\(
\frac{1}{\sqrt[|\bm{y}|]{Q(\bm{y}\mid\bm{x})}},
\)
which normalizes the above definition of perplexity by the output length $|\bm{y}|$, is also often used as the definition of perplexity.
However, in this paper, we adopt Definition \ref{def:Perplexity} as a more straightforward form that can be applied to more general data formats.
\end{remark}

\noindent \textbf{Formulation of completely faithful explanations}: We consider how to define whether an explanation of the conditional probability mass function $Q_A$ determined by an AI is completely faithful.
When we say that an explanation of a function is completely faithful to the behavior of that function, the explanation should uniquely identify the behavior of that function and distinguish it from others.
In addition, since a random string is meaningless as an explanation, a completely faithful explanation should be described formally enough that, when interpreted by an appropriate formal rule fixed in advance, the concrete values of the function can be recovered.
A formal rule must be able to be written as a program, so a completely faithful explanation should be able to compute the concrete values of the function when interpreted by a computable function.
Based on this idea, we define a completely faithful explanation as follows.

\begin{definition}[Interpretation function and completely faithful explanation]
\label{def:Explanation}
Let $\Sigma$ be the alphabet natively handled by the computation model used when discussing computability.
Let $\Lambda$ be a finite set of characters used for explanations, typically characters used by humans in natural language.
Let $\Enc^{\mathsf{PF}}_\Lambda:\Lambda\twoheadrightarrow\Sigma^*$ be a prefix-free encoding function for $\Lambda$.
In particular, assume that $\Enc^{\mathsf{PF}}_\Lambda$ is injective, its image is computable, and its image is prefix-free.
Also assume that $\Enc^{\mathsf{PF}}_\Lambda$ is chosen efficiently so that
$L_\Lambda := \max\Big\{\big|\Enc^{\mathsf{PF}}_\Lambda(\lambda)\big| \ \Big| \ \lambda\in\Lambda\Big\}$
is small.
For a string $\bm e=e_0 e_1 \cdots e_{r-1}\in\Lambda^*$, define
\begin{equation}
\Enc_\Lambda(\bm e):=
\Enc^{\mathsf{PF}}_\Lambda(e_0) \cdot \Enc^{\mathsf{PF}}_\Lambda(e_1) \cdot \cdots \cdot \Enc^{\mathsf{PF}}_\Lambda(e_{r-1}).
\end{equation}

Let $\mathcal{Z}$ be an at most countable set.
Fix a computable \NewTerm{interpretation function}
$\Interpret:\subseteq \Lambda^*\times\NN\times\mathcal{Z}\computableto\QQ$
which interprets an explanation written using characters in $\Lambda$ as a conditional probability mass function under $\Enc_\Lambda$.
A string $\bm{e}\in\Lambda^*$ is a \NewTerm{completely faithful explanation} of a function $f:\mathcal{Z}\computableto\RR$ under the interpretation function $\Interpret$ if, for every $z\in\mathcal{Z}$ and every $k\in\NN$,
\[
\big|\Interpret(\bm{e},k,z)-f(z)\big|<|\Sigma|^{-k}
\]
holds.
The \NewTerm{length} of the explanation is defined as its length as a string, namely $|\bm{e}|$.
\end{definition}

\begin{example}[Example of $\Lambda$]
As an example, one can take $\Lambda$ to be the set of characters included in Unicode, and take $\Enc_\Lambda$ to be an encoding method such as UTF-8.
In this case, $L_\Lambda=32$.
If only ASCII characters are used, then $L_\Lambda=7$.
\end{example}

\noindent \textbf{Formulation of the complexity of the AI operation environment by Kolmogorov complexity}: In quantifying the complexity of the operation environment of AI, it is information-theoretically natural to use Kolmogorov complexity \cite{Solomonoff1964PartI, Solomonoff1964PartII, Kolmogorov1965ThreeApproaches, Chaitin1969SimplicitySpeed}.
However, it is not obvious which object's Kolmogorov complexity should be used.
When the input-output relation is stochastic, one might naively want to consider the Kolmogorov complexity of the conditional probability mass function.
However, as rigorously discussed in Appendix \ref{sec:NaiveResult}, the Kolmogorov complexity of the conditional probability mass function is not directly related to AI performance and the interpretability of AI explanations.
Therefore, we use conditional Kolmogorov complexity as the indicator.

\begin{definition}[Conditional plain Kolmogorov complexity]
Fix an alphabet $\Sigma$ and a universal conditional function $U$ under a pairing function $\langle \bullet, \bullet \rangle$. For their definitions, see Appendix \ref{sec:Preliminary}.
Let $\mathcal{X}$ and $\mathcal{Y}$ be at most countable sets, and fix encoding functions
$\Enc_{\mathcal{X}}:\mathcal{X}\twoheadrightarrow\Sigma^*
\quad\text{and}\quad
\Enc_{\mathcal{Y}}:\mathcal{Y}\twoheadrightarrow\Sigma^*$
for them.
For $x\in\mathcal{X}$ and $y\in\mathcal{Y}$, the conditional plain Kolmogorov complexity $C_U(y\mid x)\in\NN$ is defined by
\begin{equation}
C_U(y\mid x)
:=
\min\{|\bm{p}| \mid U(\langle \Enc_{\mathcal{X}}(x), \bm{p}\rangle)\downarrow=\Enc_{\mathcal{Y}}(y)\}.
\end{equation}
\end{definition}

\begin{remark}[Plain Kolmogorov complexity and prefix-free Kolmogorov complexity]
The above $C_U(f)$ is called plain Kolmogorov complexity, whereas in information theory, prefix-free Kolmogorov complexity, denoted by $K_U(f)$, is often used because it has better properties.
These quantities are closely related, but they are different, and this should be noted.
\end{remark}

\section{Fundamental Quadrilemma and its implications}
\label{sec:MainResult}

In this section, we state the main theorem of this paper, which gives a quantitative relation among the expected conditional perplexity of stochastic generative AI, the expectation of conditional Kolmogorov complexity, and computer-interpretable explanations, and then explain its meaning.
As already stated, the behavior of a stochastic input-output AI $A$ is completely described by the conditional probability mass function $Q_A$ that it determines.
Therefore, it suffices to show an inequality expressing the relation between the expected logarithmic perplexity of $Q_A$, which represents the performance of $A$, and the length of an explanation of $Q_A$, which is a string that can reproduce the behavior of $A$.
This is obtained as follows.

\begin{theorem}[Fundamental Quadrilemma]
\label{thm:Quadrilemma}
Fix a universal Turing machine for defining plain Kolmogorov complexity, and let $\Sigma$ be the alphabet it natively handles.
Fix one interpretation function whose character set is $\Lambda$, and let $L_\Lambda$ be the maximum code length when encoding $\Lambda$.
Then there exists a constant $c$ such that, for every stochastic input-output AI $A$ and the conditional probability mass function $Q_A(\cdot\mid\cdot):\mathcal{X}\times\mathcal{Y}\computableto[0,1]$ that it determines, the following two statements hold.

(i) For every $x\in\mathcal{X}$, $y\in\mathcal{Y}$, and every $\bm e_x\in\Lambda^*$ that is a completely faithful explanation of the probability mass function $Q_A(\cdot\mid x)$, the following holds:
\begin{equation}
-\log Q_A(y\mid x)
+
L_\Lambda|\bm e_x|
+
2 \log(L_\Lambda|\bm e_x|+1)
\ge
C_U(y\mid x)+c.
\end{equation}

(ii) Suppose that for every $x\in\mathcal{X}$, a completely faithful explanation $\bm e_x\in\Lambda^*$ of $Q_A(\cdot\mid x)$ is given.
For every probability distribution $P\in\mathcal{P}(\mathcal{X}\times\mathcal{Y})$, the following holds:
\begin{equation}
\EE_{X,Y\sim P}[-\log Q_A(Y\mid X)]
+
\EE_{X,Y\sim P}\left[
L_\Lambda|\bm e_X|
+
2 \log(L_\Lambda|\bm e_X|+1) 
\right]
\ge
\EE_{X,Y\sim P}C_U(Y\mid X)+c.
\end{equation}

Note that $c$ does not depend on $A$, $x$, or $y$.
\end{theorem}

\begin{remark}[Interpretation of the Fundamental Quadrilemma theorem]
Theorem \ref{thm:Quadrilemma} suggests that it is impossible to achieve all of the following simultaneously.
\begin{itemize}
\item The AI operation environment has a complex input-output relation, that is, $\mathbb{E}_{X,Y\sim P}C_U(Y\mid X)$ is large.
\item AI performance is high, that is, the expected perplexity $\mathbb{E}_{X,Y\sim P}[-\log Q(Y\mid X)]$ is small.
\item A completely faithful explanation $\bm{e}_{x}$ of the AI is given for all $x \in \mathcal{X}$.
\item The AI explanation is short, that is, $\mathbb{E}_{X,Y\sim P}|\bm{e}_X|$ is small.
\end{itemize}
Note that the theorem states about local explanations $e_x$ for all $x \in \mathcal{X}$, a global explanation (an explanation of $Q$, instead of $Q(\cdot \mid x)$) must be longer than any $e_x$ up to constant, so the quadrilemma holds for a global explanation.

Some might suspect that $\EE_{X,Y\sim P}[-\log Q_A(Y\mid X)]$ may be close to $\EE_{X,Y\sim P}C_U(Y\mid X)$, resulting in no implications on $|\bm{e}_x|$, but it is not generally true. See Proposition \ref{prp:KCEntropyGap} in Appendix for details.

Although the constraint appears to become weaker when $L_\Lambda$ is large, the fact that $L_\Lambda$ cannot be made small means that the number of elements of the character set $\Lambda$ is large, and therefore the cognitive efficiency per character for humans decreases. Thus, in substance, this does not relax the constraint.
\end{remark}

\begin{remark}[Implications of the Quadrilemma theorem]
The Quadrilemma theorem claims the necessity of giving up one of the four conditions.
In environments where large-scale AI is applied, the complexity of the input-output relation is usually unavoidable.
Moreover, one cannot sacrifice the shortness of explanations, which is a necessary condition for interpretability.
Therefore, one must give up either the goodness of AI performance or the complete faithfulness of explanations.
In the usual situation, where one pursues good AI performance, it is necessary to give up the complete faithfulness of explanations.
Rather, the field of AI explanation should proceed by pursuing useful explanations that are not completely faithful.
In exceptional situations where a surrogate model is acceptable, AI performance is sacrificed, and therefore complete faithfulness of AI explanations may be achievable. This is also part of the implication of the Quadrilemma theorem.
\end{remark}

\begin{remark}[Relationship between explanation length and interpretability]
This paper is based on the premise that an excessively long explanation is not interpretable, regardless of its content.
For a quantitative discussion, see Appendix \ref{sec:HumanLimitation}.

One might point out that even a long explanation could be interpretable if it is regular.
For example, an explanation obtained by concatenating the length-two string \texttt{01} one million times would be interpretable even though it is long.
However, in such a case, there should exist an extremely short equivalent explanation that uses the regularity, such as
\[
\texttt{ans = ""; for i in range(10 ** 6): ans = ans + "01"; return ans}.
\]
According to the Fundamental Quadrilemma, as long as the other three conditions are satisfied, no short explanation exists.
Therefore, the existence of a long but regular and interpretable explanation is also ruled out.
\end{remark}
\begin{remark}[Necessity of considering the complexity of the AI operation environment and the difficulty of computing its indicator]
It is inevitable that the complexity of the AI operation environment appears in the quadrilemma.
If the true input-output relation is the identity function, then an AI with the best performance and a completely faithful and interpretable explanation of its behavior are trivially obtained.
On the other hand, when the operation environment is complex, computing a quantity that characterizes the complexity of that environment also becomes extremely difficult, and is practically impossible.
The quantity used as the complexity indicator in this paper, $\EE_{X,Y\sim P}C_U(Y\mid X)$, also has little hope of being computable in practice. Separately from this, Kolmogorov complexity also has a computability problem.
Since the environment itself is complex, computing its complexity is, in principle, a difficult problem.
Nevertheless, Theorem \ref{thm:Quadrilemma} is suggestive.
If an AI that achieves the true input-output relation and is written by a short program were realized, then $\EE_{X,Y\sim P}C_U(Y\mid X)$ would be small, and the quadrilemma would not be a problem.
However, considering the history in which text generation could not achieve sufficient performance until the advent of large-scale language models, one can formulate the hypothesis that such a situation will not occur.
This hypothesis, however, is not something that can be confirmed by numerical computation, and it will likely be shown empirically as AI develops in the future.
\end{remark}

\section{Conclusion, limitations of this study and future work}
\label{sec:Limitation}
In this study, we mathematically showed that there exists a quadrilemma stating that AI and its explanation cannot simultaneously satisfy the following four conditions: \textbf{the complexity of the environment in which the AI operates}, \textbf{the goodness of AI performance}, \textbf{the interpretability of the AI explanation}, and \textbf{the complete faithfulness of the AI explanation}.
This quadrilemma suggests that, in most applications where we cannot change the environment or sacrifice good AI performance and an interpretable explanation, we should give up complete faithfulness of explanations.

Despite our theorem's significance, this study has the following limitations, which provide promising directions for future work.

\noindent \textbf{Quantification of faithfulness of explanations}: This study treats only the case where explanations are completely faithful to AI behavior.
Relaxing this condition, quantifying the faithfulness of explanations that are not completely faithful, and deriving limitations on such explanations may be useful and interesting future work.

\noindent \textbf{Time-complexity constraints}: This study used Kolmogorov complexity that takes into account only computability, and did not take time-complexity constraints into account.
Recent models based on Transformers \cite{vaswani2017attention} and models using diffusion strategies \cite{rombach2022high} share the feature that the same parameters are reused many times.
In other words, they have many parameters, but their time complexity is even larger.
Based on this observation, it may be interesting future work to consider limitations of explainability when the true input-output relation has a lower bound on time complexity.

\newpage

\bibliographystyle{ieeetr}
\bibliography{ref}

@techreport{nist2024genai_profile,
  author      = {{National Institute of Standards and Technology}},
  title       = {Artificial Intelligence Risk Management Framework: Generative Artificial Intelligence Profile},
  institution = {National Institute of Standards and Technology},
  number      = {NIST AI 600-1},
  year        = {2024},
  doi         = {10.6028/NIST.AI.600-1},
  url         = {https://nvlpubs.nist.gov/nistpubs/ai/NIST.AI.600-1.pdf}
}

@techreport{edps2025airiskmanagement,
  author      = {{European Data Protection Supervisor}},
  title       = {Guidance for Risk Management of Artificial Intelligence systems},
  institution = {European Data Protection Supervisor},
  year        = {2025},
  url         = {https://www.edps.europa.eu/system/files/2025-11/2025-11-11_ai_risks_management_guidance_en.pdf}
}

@inproceedings{ribeiro2016should,
  title={" Why should i trust you?" Explaining the predictions of any classifier},
  author={Ribeiro, Marco Tulio and Singh, Sameer and Guestrin, Carlos},
  booktitle={Proceedings of the 22nd ACM SIGKDD international conference on knowledge discovery and data mining},
  pages={1135--1144},
  year={2016}
}

@article{lundberg2017unified,
  title={A unified approach to interpreting model predictions},
  author={Lundberg, Scott M and Lee, Su-In},
  journal={Advances in neural information processing systems},
  volume={30},
  year={2017}
}

@inproceedings{sundararajan2017axiomatic,
  title={Axiomatic attribution for deep networks},
  author={Sundararajan, Mukund and Taly, Ankur and Yan, Qiqi},
  booktitle={International conference on machine learning},
  pages={3319--3328},
  year={2017},
  organization={PMLR}
}

@article{ahern2019normlime,
  title={NormLime: A new feature importance metric for explaining deep neural networks},
  author={Ahern, Isaac and Noack, Adam and Guzman-Nateras, Luis and Dou, Dejing and Li, Boyang and Huan, Jun},
  journal={arXiv preprint arXiv:1909.04200},
  year={2019}
}

@inproceedings{ribeiro2018anchors,
  title={Anchors: High-precision model-agnostic explanations},
  author={Ribeiro, Marco Tulio and Singh, Sameer and Guestrin, Carlos},
  booktitle={Proceedings of the AAAI conference on artificial intelligence},
  volume={32},
  number={1},
  year={2018}
}

@article{lundberg2020local,
  title={From local explanations to global understanding with explainable AI for trees},
  author={Lundberg, Scott M and Erion, Gabriel and Chen, Hugh and DeGrave, Alex and Prutkin, Jordan M and Nair, Bala and Katz, Ronit and Himmelfarb, Jonathan and Bansal, Nisha and Lee, Su-In},
  journal={Nature machine intelligence},
  volume={2},
  number={1},
  pages={56--67},
  year={2020},
  publisher={Nature Publishing Group UK London}
}

@article{setzu2021glocalx,
  title={Glocalx-from local to global explanations of black box ai models},
  author={Setzu, Mattia and Guidotti, Riccardo and Monreale, Anna and Turini, Franco and Pedreschi, Dino and Giannotti, Fosca},
  journal={Artificial Intelligence},
  volume={294},
  pages={103457},
  year={2021},
  publisher={Elsevier}
}

@article{huang2023can,
  title={Can large language models explain themselves? a study of llm-generated self-explanations},
  author={Huang, Shiyuan and Mamidanna, Siddarth and Jangam, Shreedhar and Zhou, Yilun and Gilpin, Leilani H},
  journal={arXiv preprint arXiv:2310.11207},
  year={2023}
}

@article{kroeger2023context,
  title={In-context explainers: Harnessing llms for explaining black box models},
  author={Kroeger, Nicholas and Ley, Dan and Krishna, Satyapriya and Agarwal, Chirag and Lakkaraju, Himabindu},
  journal={arXiv preprint arXiv:2310.05797},
  year={2023}
}

@inproceedings{jie2024interpretable,
  title={How interpretable are reasoning explanations from prompting large language models?},
  author={Jie, Yeo Wei and Satapathy, Ranjan and Goh, Rick and Cambria, Erik},
  booktitle={Findings of the Association for Computational Linguistics: NAACL 2024},
  pages={2148--2164},
  year={2024}
}

@article{visani2024gleams,
  title={Gleams: Bridging the gap between local and global explanations},
  author={Visani, Giorgio and Stanzione, Vincenzo and Garreau, Damien},
  journal={arXiv preprint arXiv:2408.05060},
  year={2024}
}

@article{janizek2021explaining,
  title={Explaining explanations: Axiomatic feature interactions for deep networks},
  author={Janizek, Joseph D and Sturmfels, Pascal and Lee, Su-In},
  journal={Journal of Machine Learning Research},
  volume={22},
  number={104},
  pages={1--54},
  year={2021}
}

@article{wachter2017counterfactual,
  title={Counterfactual explanations without opening the black box: Automated decisions and the GDPR},
  author={Wachter, Sandra and Mittelstadt, Brent and Russell, Chris},
  journal={Harv. JL \& Tech.},
  volume={31},
  pages={841},
  year={2017},
  publisher={HeinOnline}
}

@inproceedings{karimi2021algorithmic,
  title={Algorithmic recourse: from counterfactual explanations to interventions},
  author={Karimi, Amir-Hossein and Sch{\"o}lkopf, Bernhard and Valera, Isabel},
  booktitle={Proceedings of the 2021 ACM conference on fairness, accountability, and transparency},
  pages={353--362},
  year={2021}
}

@inproceedings{selvaraju2017grad,
  title={Grad-cam: Visual explanations from deep networks via gradient-based localization},
  author={Selvaraju, Ramprasaath R and Cogswell, Michael and Das, Abhishek and Vedantam, Ramakrishna and Parikh, Devi and Batra, Dhruv},
  booktitle={Proceedings of the IEEE international conference on computer vision},
  pages={618--626},
  year={2017}
}

@inproceedings{chattopadhay2018grad,
  title={Grad-cam++: Generalized gradient-based visual explanations for deep convolutional networks},
  author={Chattopadhay, Aditya and Sarkar, Anirban and Howlader, Prantik and Balasubramanian, Vineeth N},
  booktitle={2018 IEEE winter conference on applications of computer vision (WACV)},
  pages={839--847},
  year={2018},
  organization={IEEE}
}

@inproceedings{wang2020score,
  title={Score-CAM: Score-weighted visual explanations for convolutional neural networks},
  author={Wang, Haofan and Wang, Zifan and Du, Mengnan and Yang, Fan and Zhang, Zijian and Ding, Sirui and Mardziel, Piotr and Hu, Xia},
  booktitle={Proceedings of the IEEE/CVF conference on computer vision and pattern recognition workshops},
  pages={24--25},
  year={2020}
}

@article{craven1995extracting,
  title={Extracting tree-structured representations of trained networks},
  author={Craven, Mark and Shavlik, Jude},
  journal={Advances in neural information processing systems},
  volume={8},
  year={1995}
}

@inproceedings{setiono1995understanding,
  title={Understanding neural networks via rule extraction},
  author={Setiono, Rudy and Liu, Huan},
  booktitle={IJCAI},
  volume={1},
  pages={480--485},
  year={1995}
}

@inproceedings{jacobs2022ai,
  title={Ai/ml for network security: The emperor has no clothes},
  author={Jacobs, Arthur S and Beltiukov, Roman and Willinger, Walter and Ferreira, Ronaldo A and Gupta, Arpit and Granville, Lisandro Z},
  booktitle={Proceedings of the 2022 ACM SIGSAC Conference on Computer and Communications Security},
  pages={1537--1551},
  year={2022}
}

@article{fu2020axiom,
  title={Axiom-based grad-cam: Towards accurate visualization and explanation of cnns},
  author={Fu, Ruigang and Hu, Qingyong and Dong, Xiaohu and Guo, Yulan and Gao, Yinghui and Li, Biao},
  journal={arXiv preprint arXiv:2008.02312},
  year={2020}
}

@article{buono2024expected,
  title={Expected grad-CAM: Towards gradient faithfulness},
  author={Buono, Vincenzo and Mashhadi, Peyman Sheikholharam and Rahat, Mahmoud and Tiwari, Prayag and Byttner, Stefan},
  journal={arXiv preprint arXiv:2406.01274},
  year={2024}
}

@article{kumarage2026explainable,
  title={Explainable Generative AI: A Two-Stage Review of Existing Techniques and Future Research Directions},
  author={Kumarage, Prabha M and Saarela, Mirka},
  journal={AI},
  volume={7},
  number={1},
  pages={31},
  year={2026},
  publisher={MDPI}
}

@article{arrieta2020explainable,
  title={Explainable Artificial Intelligence (XAI): Concepts, taxonomies, opportunities and challenges toward responsible AI},
  author={Arrieta, Alejandro Barredo and D{\'\i}az-Rodr{\'\i}guez, Natalia and Del Ser, Javier and Bennetot, Adrien and Tabik, Siham and Barbado, Alberto and Garc{\'\i}a, Salvador and Gil-L{\'o}pez, Sergio and Molina, Daniel and Benjamins, Richard and others},
  journal={Information fusion},
  volume={58},
  pages={82--115},
  year={2020},
  publisher={Elsevier}
}

@online{guinness_pi_places_memorised,
  author       = {{Guinness World Records}},
  title        = {Most decimal places of Pi memorized},
  year         = {2015},
  url          = {https://www.guinnessworldrecords.com/world-records/most-pi-places-memorised},
  urldate      = {2026-05-04},
  note         = {Record achieved by Rajveer Meena at VIT University, Vellore, India, on 21 March 2015}
}

@article{brysbaert2019many,
  title={How many words do we read per minute? A review and meta-analysis of reading rate},
  author={Brysbaert, Marc},
  journal={Journal of memory and language},
  volume={109},
  pages={104047},
  year={2019},
  publisher={Elsevier}
}

@article{radford2019language,
  title  = {Language Models are Unsupervised Multitask Learners},
  author = {Radford, Alec and Wu, Jeffrey and Child, Rewon and Luan, David and Amodei, Dario and Sutskever, Ilya},
  year   = {2019},
  journal = {OpenAI Blog},
  url    = {https://cdn.openai.com/better-language-models/language_models_are_unsupervised_multitask_learners.pdf}
}

@inproceedings{rombach2022high,
  title={High-resolution image synthesis with latent diffusion models},
  author={Rombach, Robin and Blattmann, Andreas and Lorenz, Dominik and Esser, Patrick and Ommer, Bj{\"o}rn},
  booktitle={Proceedings of the IEEE/CVF conference on computer vision and pattern recognition},
  pages={10684--10695},
  year={2022}
}

@article{turpin2023language,
  title={Language models don't always say what they think: Unfaithful explanations in chain-of-thought prompting},
  author={Turpin, Miles and Michael, Julian and Perez, Ethan and Bowman, Samuel},
  journal={Advances in Neural Information Processing Systems},
  volume={36},
  pages={74952--74965},
  year={2023}
}

@article{zhao2024explainability,
  title={Explainability for large language models: A survey},
  author={Zhao, Haiyan and Chen, Hanjie and Yang, Fan and Liu, Ninghao and Deng, Huiqi and Cai, Hengyi and Wang, Shuaiqiang and Yin, Dawei and Du, Mengnan},
  journal={ACM Transactions on Intelligent Systems and Technology},
  volume={15},
  number={2},
  pages={1--38},
  year={2024},
  publisher={ACM New York, NY}
}

@article{vaswani2017attention,
  title={Attention is all you need},
  author={Vaswani, Ashish and Shazeer, Noam and Parmar, Niki and Uszkoreit, Jakob and Jones, Llion and Gomez, Aidan N and Kaiser, {\L}ukasz and Polosukhin, Illia},
  journal={Advances in neural information processing systems},
  volume={30},
  year={2017}
}

@book{cutland1980computability,
  title={Computability: An introduction to recursive function theory},
  author={Cutland, Nigel},
  year={1980},
  publisher={Cambridge university press}
}

@book{li2008introduction,
  title={An introduction to Kolmogorov complexity and its applications},
  author={Li, Ming and Vit{\'a}nyi, Paul and others},
  volume={3},
  year={2008},
  publisher={Springer}
}

@article{Solomonoff1964PartI,
  author  = {Solomonoff, Ray J.},
  title   = {A Formal Theory of Inductive Inference. Part I},
  journal = {Information and Control},
  volume  = {7},
  number  = {1},
  pages   = {1--22},
  month   = mar,
  year    = {1964},
  doi     = {10.1016/S0019-9958(64)90223-2}
}

@article{Solomonoff1964PartII,
  author  = {Solomonoff, Ray J.},
  title   = {A Formal Theory of Inductive Inference. Part II},
  journal = {Information and Control},
  volume  = {7},
  number  = {2},
  pages   = {224--254},
  month   = jun,
  year    = {1964},
  doi     = {10.1016/S0019-9958(64)90131-7}
}

@article{Kolmogorov1965ThreeApproaches,
  author  = {Kolmogorov, A. N.},
  title   = {Three Approaches to the Quantitative Definition of Information},
  journal = {Problems of Information Transmission},
  volume  = {1},
  number  = {1},
  pages   = {1--7},
  year    = {1965}
}

@article{Chaitin1969SimplicitySpeed,
  author  = {Chaitin, Gregory J.},
  title   = {On the Simplicity and Speed of Programs for Computing Infinite Sets of Natural Numbers},
  journal = {Journal of the ACM},
  volume  = {16},
  number  = {3},
  pages   = {407--422},
  month   = jul,
  year    = {1969},
  doi     = {10.1145/321526.321530}
}

@article{zhang2023trade,
  title={Trade-off between efficiency and consistency for removal-based explanations},
  author={Zhang, Yifan and He, Haowei and Tan, Zhiquan and Yuan, Yang},
  journal={Advances in Neural Information Processing Systems},
  volume={36},
  pages={25627--25661},
  year={2023}
}

@inproceedings{bressan2024theory,
  title={A theory of interpretable approximations},
  author={Bressan, Marco and Cesa-Bianchi, Nicol{\`o} and Esposito, Emmanuel and Mansour, Yishay and Moran, Shay and Thiessen, Maximilian},
  booktitle={The Thirty Seventh Annual Conference on Learning Theory},
  pages={648--668},
  year={2024},
  organization={PMLR}
}

@article{bilodeau2024impossibility,
  title={Impossibility theorems for feature attribution},
  author={Bilodeau, Blair and Jaques, Natasha and Koh, Pang Wei and Kim, Been},
  journal={Proceedings of the National Academy of Sciences},
  volume={121},
  number={2},
  pages={e2304406120},
  year={2024},
  publisher={National Academy of Sciences}
}

@inproceedings{frost2024partially,
  title={Partially interpretable models with guarantees on coverage and accuracy},
  author={Frost, Nave and Lipton, Zachary and Mansour, Yishay and Moshkovitz, Michal},
  booktitle={International conference on algorithmic learning theory},
  pages={590--613},
  year={2024},
  organization={PMLR}
}

@article{shportko2026kolmogorov,
  title={Kolmogorov Complexity Bounds for LLM Steganography and a Perplexity-Based Detection Proxy},
  author={Shportko, Andrii},
  journal={arXiv preprint arXiv:2603.21567},
  year={2026}
}

@article{burden2025conversational,
  title   = {Conversational complexity for assessing risk in large language models},
  author  = {Burden, John and Cebrian, Manuel and Hernandez-Orallo, Jose},
  journal = {EPJ Data Science},
  volume  = {14},
  number  = {78},
  year    = {2025},
  doi     = {10.1140/epjds/s13688-025-00592-4},
  url     = {https://doi.org/10.1140/epjds/s13688-025-00592-4}
}

@inproceedings{pan2025understanding,
  title     = {Understanding LLM Behaviors via Compression: Data Generation, Knowledge Acquisition and Scaling Laws},
  author    = {Pan, Zhixuan and Wang, Shaowen and Liao, Pengfei and Li, Jian},
  booktitle = {Advances in Neural Information Processing Systems},
  volume    = {38},
  year      = {2025},
  note      = {Spotlight},
  url       = {https://openreview.net/forum?id=853SwC2dMZ}
}

@inproceedings{deletang2024language,
  title     = {Language Modeling Is Compression},
  author    = {Del{\'e}tang, Gr{\'e}goire and Ruoss, Anian and Duquenne, Paul-Ambroise and Catt, Elliot and Genewein, Tim and Mattern, Christopher and Grau-Moya, Jordi and Wenliang, Li Kevin and Aitchison, Matthew and Orseau, Laurent and Hutter, Marcus and Veness, Joel},
  booktitle = {The Twelfth International Conference on Learning Representations},
  year      = {2024},
  url       = {https://openreview.net/forum?id=jznbgiynus}
}

@inproceedings{elmoznino2025incontext,
  title     = {In-Context Learning and Occam's Razor},
  author    = {Elmoznino, Eric and Marty, Tom and Kasetty, Tejas and Gagnon, Leo and Mittal, Sarthak and Fathi, Mahan and Sridhar, Dhanya and Lajoie, Guillaume},
  booktitle = {Proceedings of the 42nd International Conference on Machine Learning},
  pages     = {15296--15319},
  year      = {2025},
  editor    = {Singh, Aarti and Fazel, Maryam and Hsu, Daniel and Lacoste-Julien, Simon and Berkenkamp, Felix and Maharaj, Tegan and Wagstaff, Kiri and Zhu, Jerry},
  volume    = {267},
  series    = {Proceedings of Machine Learning Research},
  month     = {13--19 Jul},
  publisher = {PMLR},
  url       = {https://proceedings.mlr.press/v267/elmoznino25b.html}
}

@article{grunwald2008algorithmic,
  title={Algorithmic information theory},
  author={Gr{\"u}nwald, Peter D and Vit{\'a}nyi, Paul MB and others},
  journal={Handbook of the Philosophy of Information},
  pages={281--320},
  year={2008},
  publisher={Amsterdam, Netherlands: Elsevier}
}

@article{bengio2003neural,
  title={A neural probabilistic language model},
  author={Bengio, Yoshua and Ducharme, R{\'e}jean and Vincent, Pascal and Jauvin, Christian},
  journal={Journal of machine learning research},
  volume={3},
  number={Feb},
  pages={1137--1155},
  year={2003}
}

@inproceedings{mikolov2010recurrent,
  title={Recurrent neural network based language model.},
  author={Mikolov, Tomas and Karafi{\'a}t, Martin and Burget, Lukas and Cernock{\`y}, Jan and Khudanpur, Sanjeev},
  booktitle={Interspeech},
  volume={2},
  number={3},
  pages={1045--1048},
  year={2010},
  organization={Makuhari}
}

\newpage

\appendix

\section{The number of parameters in generative AI exceeds human processing capacity}

\label{sec:HumanLimitation}

This section provides specific discussions on the human capacity to recognize a series of letters and discusses why it matters. For example, for open-weight models, the source code and checkpoint can be regarded as explanations that completely faithfully describe the behavior of the model.
However, even if a small 7B-parameter model is boldly quantized to 8 bits, the parameters amount to 56 billion bits, which exceeds human memory capacity \cite{guinness_pi_places_memorised} and reading capacity \cite{brysbaert2019many}.
From the above consideration, we see that the shortness of explanations is sufficient as a necessary condition for interpretability of explanations to exclude trivial explanations based on checkpoints.
Regarding reading speed, for English, there exists a statistical result of 238 words per minute for nonfiction text, with 4.6 characters per word, which corresponds to approximately 4 billion bits per year.
If a model with 1B parameters, which can be regarded as having a small number of parameters as of 2026, uses the standard bf16 format, then it has 16 billion bits. Therefore, even reading through the checkpoint parameters is extremely difficult for humans.
If one requires not only reading but also memorizing, it becomes even more difficult. The world record for memorizing the decimal expansion of pi, which is data in which regularities are difficult to find, is 70,000 digits \cite{guinness_pi_places_memorised}, which is at most 233 thousand bits.
This suggests that the source code and checkpoints, namely the weight parameters, of open-weight models such as the Llama, Qwen, DeepSeek, and GLM series are complete explanations, but humans cannot memorize them, nor even read through them. Thus, they are not interpretable.
This is because even a 7B model, which can be regarded as extremely small as of 2026, becomes 28 billion bits even after 4-bit quantization, which far exceeds the amount that a human could read even by reading 24 hours a day throughout an entire year.

\section{Definitions of Computability and Kolmogorov Complexity}
\label{sec:Preliminary}

\subsection{Partial functions and computability}

We define several concepts concerning computable functions.
For details, see \cite{cutland1980computability} for computability including multiple computation models, and \cite{li2008introduction, grunwald2008algorithmic} for Kolmogorov complexity. This paper basically defines Kolmogorov complexity following the style of \cite{grunwald2008algorithmic}, but note that whereas \cite{grunwald2008algorithmic} directly formulates it for functions with natural-number inputs and outputs, this paper formulates it for functions with string inputs and outputs. These can be regarded as equivalent by using a standard bijection between the set of natural numbers and the set of finite strings.
\begin{definition}[Computable function]
Let $\Sigma$ be a finite set with at least two elements.
A function $f$ is a partial computable function with source set $\mathcal{X}=\Sigma^*$ and target set $\mathcal{Y}=\Sigma^*$ if there exist a set $\dom f\subseteq\mathcal{X}$, called the domain, and a function $M$ implemented in some computation model such that the following hold.
\begin{itemize}
\item If $\bm{x}\in\dom f$, then $f(\bm{x})\in\mathcal{Y}$ is defined, and when $M$ receives $\bm{x}$ as input, it outputs $\bm{y}=f(\bm{x})$ and halts.
\item If $\bm{x}\in\mathcal{X}\setminus\dom f$, then $f(\bm{x})$ is undefined, and when $M$ receives $\bm{x}$ as input, it does not halt.
\end{itemize}
In this case, we write $f:\subseteq\mathcal{X}\computableto\mathcal{Y}$, and write the domain $\mathcal{X}'$ as $\dom f$.
Moreover, we write $f(\bm{x})\downarrow$ to mean $\bm{x}\in\dom f$, and write $f(\bm{x})\downarrow=\bm{y}$ to mean that $f(\bm{x})\downarrow$ and $f(\bm{x})=\bm{y}$.
When $\src f$ denotes the source domain of $f$, we write $f(\bm{x})\uparrow$ to mean $\bm{x}\in\src f\setminus\dom f$.
If $\dom f=\src f$, then $f$ is called a total computable function, and we write $f:\mathcal{X}\computableto\mathcal{Y}$.
\end{definition}
The computability defined above is known to be equivalent no matter which of the usual sufficiently powerful computation models is considered. This is the Church--Turing thesis.
Computability is standardly defined using Turing machines, but note that an equivalent definition is obtained by using a standard programming language, e.g., C, Python, etc., when working with unlimited time and memory.
Here, $f (\bm{x}) \uparrow$ corresponds to the situation when the function $f$ implemented in such a programming language does not halt owing to, e.g., an infinite loop.
Since AIs are a composition of functions implemented in those programming languages, we can say that AIs are computable functions in the sense of the above definition.

\begin{definition}[Equivalence of partial computable functions]
Fix the source set $\mathcal{X}=\Sigma^*$ and the target set $\mathcal{Y}=\Sigma^*$.
For two partial computable functions $f, g: \subseteq \mathcal{X} \computableto \mathcal{Y}$, define the equality $f=g$ to mean that $\dom f = \dom g$ and that, for every $x \in \dom f$, $f (x) \downarrow = g (x) \downarrow$.
\end{definition}

The computability, or decidability, of a subset of the string set $\Sigma^*$ is defined by the computability of its indicator function.

\begin{definition}[Computability of a string set]
A subset $\mathcal{A}\subseteq\Sigma^*$ is \NewTerm{computable}, or \NewTerm{decidable}, if there exists a total computable function $\chi:\Sigma^*\to\Sigma^*$ such that
\begin{equation}
\bm{x}\in\mathcal{A}\Leftrightarrow \chi(\bm{x})\ne ()
\end{equation}
holds.
\end{definition}

\subsection{Treatment of general functions}

So far, we have defined computability for functions that take strings as inputs and return strings as outputs.
For functions between general sets, computability can also be discussed through encodings that convert their elements into strings.
However, since the set of all strings is countable, only at most countable sets can be handled.
Functions whose target set is the entire set of real numbers are defined later.

First, it is necessary to map the elements of the sets under consideration to strings.
\begin{definition}[Encoding function]
For an at most countable set $\mathcal{Z}$, a function $\Enc_{\mathcal{Z}}: \mathcal{Z} \to \Sigma^*$ is an \NewTerm{encoding function} if $\Enc_{\mathcal{Z}}$ is injective and its image $\Enc_{\mathcal{Z}} (\mathcal{Z}) := \{\Enc_{\mathcal{Z}} \mid z \in \mathcal{Z}\}$ is a computable, or decidable, set.
\end{definition}
Here, the image being computable, or decidable, means that it is possible to decide by computation whether a given string is a valid code.

\begin{definition}[Computability of general functions]
\label{def:GeneralComputability}
Let $\Sigma$ be a finite set with at least two elements, let $\mathcal{X}$ and $\mathcal{Y}$ be at most countable sets, and fix encoding functions
$\Enc_{\mathcal{X}}:\mathcal{X}\twoheadrightarrow\Sigma^*$,
$\Enc_{\mathcal{Y}}:\mathcal{Y}\twoheadrightarrow\Sigma^*$
for them. These functions are injective, and their images are computable, or decidable.
Then a function $f$ is a \NewTerm{partial computable function} with source set $\mathcal{X}$ and target set $\mathcal{Y}$ if there exists a set $\dom f\subseteq\mathcal{X}$, called the domain, such that $f (x) \in \mathcal{Y}$ is defined for every $x \in \dom f$, and there exists a partial computable function on strings $\tilde{f}: \subseteq \Sigma^* \computableto \Sigma^*$ such that the equality of partial computable functions $\tilde{f} \big(\Enc_{\mathcal{X}} (\bullet) \big) = \Enc_{\mathcal{Y}} \big(f (\bullet)\big)$ holds. Concretely, this means the following:
\begin{equation}
\left\{
\begin{aligned}
&\tilde{f} \big(\Enc_{\mathcal{X}} (x) \big) \downarrow = \Enc_{\mathcal{Y}} \big(f (x)\big) & \quad \text{if } x &\in \dom f, \\ 
&\tilde{f} \big(\Enc_{\mathcal{X}} (x) \big) \uparrow & \text{if } x &\in \mathcal{X} \setminus \dom f.  
\end{aligned}
\right.
\end{equation}
When $f$ is a partial computable function with source set $\mathcal{X}$ and target set $\mathcal{Y}$, we write $f:\subseteq\mathcal{X}\computableto\mathcal{Y}$.
Moreover, we write $f(x)\downarrow$ to mean $x\in\dom f$, and write $f(x)\downarrow=y$ to mean that $f(x)\downarrow$ and $f(x)=y$.
When $\src f$ denotes the source set of $f$, we write $f(x)\uparrow$ to mean $x\in\src f\setminus\dom f$.
If $\dom f=\src f$, then $f$ is called a total computable function, and we write $f:\mathcal{X}\computableto\mathcal{Y}$.
\end{definition}

\subsection{Treatment of functions with multiple inputs and outputs and real-valued functions}
In practical applications, it is natural to consider multiple inputs and outputs. Moreover, as discussed later, even when dealing with functions that are originally single-input functions, the treatment of multiple inputs is indispensable if they have real-valued outputs. Since this paper deals with probability mass functions, it is indispensable to specify how real-valued functions are handled.

On the other hand, the objects most directly handled by computation models are single-input, single-output partial functions that receive one string and return one string, so it is necessary to convert multiple variables into a single string.
The nontrivial issue in doing so is the separation between strings.
For example, suppose that distinct variable values $x_1, x_2, x_3, x_4$ are converted into strings $\bm{x}_1 = \mathtt{0}, \bm{x}_2 = \mathtt{01}, \bm{x}_3 = \mathtt{010}, \bm{x}_4 = \mathtt{1010}$. Then the concatenation $\bm{x}_2 \cdot \bm{x}_3$ of $\bm{x}_2$ and $\bm{x}_3$ and the concatenation $\bm{x}_1 \cdot \bm{x}_4$ of $\bm{x}_1$ and $\bm{x}_4$ are both $\mathtt{01010}$, so $(x_1, x_4)$ and $(x_2, x_3)$ cannot be distinguished.
This is because the delimiter position is not determined.
To avoid this situation, it suffices to fix a convention that assigns the preceding string to an element of a \NewTerm{prefix-free set}. We first define prefix-free sets.
\begin{definition}[Prefix-free set, self-delimiting encoding function, and pairing]
A set of strings $\mathcal{F} \subseteq \Sigma^*$ is \NewTerm{prefix-free} if, for every $\bm{c} \in \mathcal{F}$ and every $\bm{z} \in \Sigma^* \setminus \{()\}$, $\bm{c} \cdot \bm{z} \notin \mathcal{F}$.
A total computable function $\overline{\bullet}: \Sigma^* \computableinjection \Sigma^*$, where $\bullet$ is a placeholder, is a \NewTerm{self-delimiting encoding function} if $\overline{\bullet}$ is an injective computable function whose image is a computable, or decidable, prefix-free set.
Using a self-delimiting encoding function $\overline{\bullet}$, define the \NewTerm{pairing} $\langle \bm{x}_{0}, \bm{x}_{1}, ..., \bm{x}_{n-2}, \bm{x}_{n-1} \rangle \in \Sigma^*$ of finitely many strings $\bm{x}_{0}, \bm{x}_{1}, ..., \bm{x}_{n-2}, \bm{x}_{n-1} \in \Sigma^*$ by $\langle \bm{x}_{0}, \bm{x}_{1}, ..., \bm{x}_{n-2}, \bm{x}_{n-1} \rangle := \overline{\bm{x}_{0}} \cdot \overline{\bm{x}_{1}} \cdot \cdots \cdot \overline{\bm{x}_{n-2}} \cdot \bm{x}_{n-1}$. Note that $\overline{\bullet}$ is not applied to $\bm{x}_{n-1}$.
\end{definition}

\begin{proposition}[A self-delimiting encoding function enables unique decomposition]
The $n$-variable pairing function determined by a self-delimiting encoding function $\overline{\bullet}$ is an injection from $(\Sigma^*)$ to $\Sigma^*$. Moreover, there exists a partial computable function $\pi_i^n: \subseteq \Sigma^* \to \Sigma^*$, called the projection onto the $i$-th element, such that
\begin{equation}
\begin{cases}
\pi_i^n (\bm{z}) \downarrow = \bm{x}_i \ & \text{if $\exists! (\bm{x}_{0}, \bm{x}_{1}, ..., \bm{x}_{n-2}, \bm{x}_{n-1}) \in (\Sigma^*)^n, \bm{z} = \langle \bm{x}_{0}, \bm{x}_{1}, ..., \bm{x}_{n-2}, \bm{x}_{n-1} \rangle$,} \\
\pi_i^n (\bm{z}) \uparrow \ & \text{otherwise}. 
\end{cases}
\end{equation}
\end{proposition}
\begin{proof}
We show the injectivity of the pairing function in the case $n = 2$.
Suppose that $\langle \bm{x}_0, \bm{x}_1 \rangle = \langle \bm{x}'_0, \bm{x}'_1 \rangle =: \bm{z}$, that is, $\overline{\bm{x}_0} \cdot \bm{x}_1 = \overline{\bm{x}'_0} \cdot \bm{x}'_1 = \bm{z}$.
Since $\overline{\bm{x}_0}$ and $\overline{\bm{x}'_0}$ both belong to a prefix-free set and are both prefixes of $\bm{z}$, we have $\overline{\bm{x}_0} = \overline{\bm{x}'_0}$.
By the injectivity of $\overline{\bullet}$, it follows that $\bm{x}_0 = \bm{x}'_0$.
Together with $\overline{\bm{x}_0} \cdot \bm{x}_1 = \overline{\bm{x}'_0} \cdot \bm{x}'_1$, this also implies $\bm{x}_1 = \bm{x}'_1$.
Thus, in the two-variable case, the pairing function is injective.
For a general $n$-variable pairing, since $\langle \bm{x}_{0}, \bm{x}_{1}, ..., \bm{x}_{n-2}, \bm{x}_{n-1} \rangle = \langle \bm{x}_{0}, \langle\bm{x}_{1}, ..., \langle \bm{x}_{n-2}, \bm{x}_{n-1} \rangle \cdots \rangle\rangle$, injectivity is easily shown by mathematical induction.
The projection functions are obtained by the following algorithm.

\noindent \textbf{Projection function from a pairing}
\begin{itemize}
\item \textbf{Input}: $\bm{z} \in \Sigma^*$.
\item \textbf{Step 1}: $\bm{t} \gets \bm{z}$.
\item \textbf{Step 2}: $j \gets 0$:
\item \textbf{Step 3}: $i \gets 1$:
\item \textbf{Step 4}: \textbf{If} $\bm{t}_{<i} := t_0 t_1 \cdots t_{i-1} \in \overline{\Sigma^*}$, \textbf{then} \textbf{go to} \textbf{Step 7}.
\item \textbf{Step 5}: \textbf{If} $i = |\bm{t}|$, \textbf{then} \textbf{go to} an infinite loop.
\item \textbf{Step 6}: $i \gets i + 1$ \textbf{then} \textbf{go to} \textbf{Step 4}.
\item \textbf{Step 7}: $\bm{x}_i := \bm{t}_{<i}, \bm{t} \gets \bm{t}_{\ge i}, j \gets j + 1$, then \textbf{go to} \textbf{Step 3}.
\item \textbf{Step 8}: $\bm{x}_{n-1} := \bm{t}$.
\item \textbf{Step 9}: \textbf{Output} $\bm{x}_i$ and \textbf{terminate}.
\end{itemize}

\end{proof}

Because of the above property, when handling multiple variables, after encoding them, one can apply a self-delimiting encoding function and construct a pairing, thereby applying the discussion of computability for functions that take strings as inputs.

\begin{definition}[Computability for multiple inputs and multiple outputs]
In the case of multiple inputs and outputs, the input and output can be regarded as elements of product sets.
More specifically, suppose that the source set is
\[
\mathcal{X}=\mathcal{X}_0\times\mathcal{X}_1\times\cdots\times\mathcal{X}_{m-1},
\]
and that, for $i=0,1,\ldots,m-1$, $\mathcal{X}_i$ is an at most countable set with an encoding function
$\Enc_i:\mathcal{X}_i\computableinjection\Sigma^*$.
Then, for $x=(x_0,x_1,\ldots,x_{m-1})\in\mathcal{X}$, define $\Enc_{\mathcal{X}}$ by
\begin{equation}
\begin{split}
\Enc_{\mathcal{X}}(x)
& := \langle \Enc_0(x_0), \Enc_1(x_1),\cdots,\Enc_{m-2}(x_{m-2}), \Enc_{m-1}(x_{m-1}) \rangle \\
& = \overline{\Enc_0(x_0)} \cdot \overline{\Enc_1(x_1)}\cdot\cdots\cdot\overline{\Enc_{m-2}(x_{m-2})}\cdot\Enc_{m-1}(x_{m-1})
\end{split}
\end{equation}
and apply Definition \ref{def:GeneralComputability}.
The case of multiple outputs is handled similarly.
\end{definition}

Next, consider real-valued functions. Since computers cannot directly handle real numbers themselves, outputting rational numbers that approximate them to arbitrary precision is regarded as equivalent to outputting real numbers.

\begin{definition}[Computability of real-valued functions]
\label{def:RealValuedComputability}
Let $\mathcal{X}$ be an at most countable set, and fix an encoding function $\Enc_{\mathcal X}:\mathcal X\twoheadrightarrow\Sigma^*$.
For a partial real-valued function $f:\subseteq\mathcal{X}\to\RR$, a partial computable function
$\tilde f:\subseteq\NN\times\mathcal{X}\computableto\QQ$
is a \NewTerm{computable arbitrary-precision approximation function} of $f$ if $\dom\tilde f=\NN\times\dom f$ and, for every $x\in\dom f$ and every $k\in\NN$,
\begin{equation}
|\tilde f(k,x)-f(x)|\le |\Sigma|^{-k}
\end{equation}
holds.
When $f$ has a computable arbitrary-precision approximation function, $f$ is called a \NewTerm{partial computable real-valued function} on the source set $\mathcal{X}$, and we write
$f:\subseteq\mathcal{X}\computableto\RR$.
The other notations follow Definition \ref{def:GeneralComputability}.
\end{definition}

\subsection{Representation of functions by strings through a universal function}
The interest of this paper lies in explaining AI.
This means obtaining a string that represents AI in some sense.
Therefore, we want to represent computable functions by strings.
This is made possible by universal functions.
Universal functions can emulate other computable functions from strings corresponding to them, namely programs.

However, in order to handle conditions correctly, we define not a universal function in the usual sense in the context of computability theory, but a universal conditional function.

\begin{definition}[Universal conditional function]
\label{def:UniversalConditional}
A partial computable function $U:\subseteq\Sigma^*\computableto\Sigma^*$ is a \NewTerm{universal conditional partial computable function} if there exists a decidable prefix-free set $\Programs_U \subseteq \Sigma^*$ such that, for every partial computable function $f:\subseteq\Sigma^*\computableto\Sigma^*$, there exists a string $\bm{p} \in \Programs_U$ for which the equality of partial functions $U (\overline{\bullet} \cdot \bm{p} \cdot \diamond) = f (\overline{\bullet} \cdot \diamond)$ holds, including their domains. More concretely, this means that, for every $\bm{x}, \bm{w} \in \Sigma^*$, the following holds:
\begin{equation}
\begin{cases}
U(\overline{\bm{x}} \cdot \bm{p} \cdot \bm{w}) \downarrow = f(\overline{\bm{x}} \cdot \bm{w}) \downarrow & \quad \text{if $\overline{\bm{x}} \cdot \bm{w} \in \dom f$}, \\
U(\overline{\bm{x}} \cdot \bm{p} \cdot \bm{w}) \uparrow & \quad \text{if $\overline{\bm{x}} \cdot \bm{w} \notin \dom f$}.
\end{cases}    
\end{equation}
\end{definition}

\begin{remark}[Difference from ordinary universal functions]
Definition \ref{def:UniversalConditional} imposes stronger conditions than the definition of universal functions in contexts other than Kolmogorov complexity.
In the usual context, a partial computable function $W:\subseteq\Sigma^*\computableto\Sigma^*$ is a \NewTerm{universal function} if, for every partial computable function $f:\subseteq\Sigma^*\computableto\Sigma^*$, there exists a string $\bm{q}\in\Sigma^*$ such that the equality of partial computable functions
$W(\overline{\bm{q}} \cdot \bullet)=f$
holds.
A universal conditional function is a universal function.
Concretely, by the partial computability of the inverse of $\overline{\bullet}$, for every $f: \subseteq \Sigma^* \computableto \Sigma^*$, a partial function $f': \subseteq \Sigma^* \computableto \Sigma^*$ satisfying $f'(\overline{\bullet}) = f$ is partial computable.
Then, by the universality of $U$, there exists $\bm{p} \in \Programs_U$ such that, for every $\bm{x}, \bm{w} \in \Sigma^*$, $U (\overline{\bm{x}} \cdot \bm{p} \cdot \bm{w}) = f' (\overline{\bm{x}} \cdot \bm{w})$.
Taking $\bm{w} = ()$, we obtain $U (\overline{\bm{x}} \cdot \bm{p}) = f' (\overline{\bm{x}})$.
By the construction of $f'$, $f' (\overline{\bm{x}}) = f (\bm{x})$, and hence the universality of $U$ in the ordinary sense follows.

Conversely, a universal function is not, in general, a universal \Emph{conditional} function.
However, if the existence of a universal function is assumed, then it is easy to construct a universal conditional function.
Concretely, when $W$ is a universal function, construct $U$ so that $U (\overline{\bm{x}} \cdot \overline{\bm{q}} \cdot \bm{w}) = W \Big(\overline{\overline{\bm{x}} \cdot \bm{w}} \cdot \bm{q}\Big)$.
By the universality of $W$, for every partial computable function, there exists $\bm{q} \in \Sigma^*$ such that $W (\overline{\bullet} \cdot \bm{q}) = f (\bullet)$.
Therefore, for every $\bm{x}, \bm{w}$, taking $\bullet = \overline{\bm{x}} \cdot \bm{w}$, we have $W \Big(\overline{\overline{\bm{x}} \cdot \bm{w}} \cdot \bm{q}\Big) = f (\overline{\bm{x}} \cdot \bm{w})$.
Thus, by the construction of $U$, $U (\overline{\bm{x}} \cdot \overline{\bm{q}} \cdot \bm{w}) = W \Big(\overline{\overline{\bm{x}} \cdot \bm{w}} \cdot \bm{q}\Big) = f (\overline{\bm{x}} \cdot \bm{w})$.
\end{remark}
Universal functions and universal conditional functions associate computable functions $f$ with strings $\bm{p}\in\Sigma^*$; note that this association is not one-to-one, but one-to-many.
The string $\bm{p}$ can be interpreted as a program implementing $f$.
However, note that $\bm{p}$ here is a string written in characters used by the computer, and natural language is not directly taken into account.
Recall that, in the main text, we introduced interpretation functions in order to take natural-language strings into account.

\subsection{Plain Kolmogorov complexity}

\begin{definition}[Plain Kolmogorov complexity of strings]

Let $U: \subseteq \Sigma^* \computableto \Sigma^*$ be a universal partial computable function.
For $\bm{x}, \bm{y} \in \Sigma^*$, define the conditional plain Kolmogorov complexity $C_U (\bm{y} \mid \bm{x}) \in \NN$ by
\begin{equation}
C_U (\bm{y} \mid \bm{x}) = \min \Big\{|\bm{p}| \ \Big| \ U (\overline{\bm{x}} \cdot \bm{p}) \downarrow = \bm{y} \Big\}.
\end{equation}
Moreover, define the unconditional plain Kolmogorov complexity of $\bm{y} \in \Sigma^*$ by $C_U (\bm{y}) := C_U \big(\bm{y} \ \big| \ () \big)$.
\end{definition}

\begin{proposition}[Invariance theorem for plain Kolmogorov complexity]
\label{prp:PlainInvariance}
Consider two universal partial computable functions $U, U': \subseteq \Sigma^* \computableto \Sigma^*$.
There exists a constant $c_{U, U'} \in \NN$ such that, for every $\bm{x}, \bm{y} \in \Sigma^*$, $C_{U'} (\bm{y} \mid \bm{x}) \le C_{U} (\bm{y} \mid \bm{x}) + c_{U', U}$.
Note that $c_{U', U}$ does not depend on $\bm{x}, \bm{y}$.
\end{proposition}

\begin{remark}[The value of Kolmogorov complexity does not essentially depend on the choice of universal conditional function]
As asserted by the invariance theorem above, even if one changes the universal conditional function in the definition of Kolmogorov complexity, the value changes only by an amount that does not depend on $\bm{x}$ or $\bm{y}$.
In this sense, the value of Kolmogorov complexity does not essentially depend on the choice of universal conditional function.    
\end{remark}

\begin{proof}[Proof of Proposition \ref{prp:PlainInvariance}]
By the universality of $U'$, there exists $\bm{p}_U \in \Programs_{U'} (U)$ such that, for every $\bm{x}, \bm{w} \in \Sigma^*$, $U'(\overline{\bm{x}} \cdot \bm{p}_U \cdot \bm{w}) = U (\overline{\bm{x}} \cdot \bm{w})$.
Moreover, by the definition of $C_U (\bm{y} \mid \bm{x})$, there exists $\bm{p}_{\bm{y} \mid \bm{x}} \in \Sigma^*$ such that $U (\overline{\bm{x}} \cdot \bm{p}_{\bm{y} \mid \bm{x}}) \downarrow = \bm{y}$ and $|\bm{p}_{\bm{y} \mid \bm{x}}| = C_U (\bm{y} \mid \bm{x})$.
Therefore, $U'(\overline{\bm{x}} \cdot \bm{p}_U \cdot \bm{p}_{\bm{y} \mid \bm{x}}) = U (\overline{\bm{x}} \cdot \bm{p}_{\bm{y} \mid \bm{x}}) = \bm{y}$.
Thus, the following holds:
\begin{equation}
C_{U'} (\bm{y} \mid \bm{x}) \le |\bm{p}_U \cdot \bm{p}_{\bm{y} \mid \bm{x}}| = |\bm{p}_U| + |\bm{p}_{\bm{y} \mid \bm{x}}| = |\bm{p}_U| + C_U (\bm{y} \mid \bm{x})
\end{equation}
Since $\bm{p}_U$ does not depend on $\bm{x}, \bm{y}$, the proof is complete by taking $c_{U', U} = |\bm{p}_U|$.
\end{proof}

\begin{definition}[Plain Kolmogorov complexity of functions]
For at most countable sets $\mathcal{X}, \mathcal{Y}, \mathcal{Z}$, fix encoding functions $\Enc_{\mathcal{X}}:\mathcal{X}\twoheadrightarrow \Sigma^*, \Enc_{\mathcal{Y}}:\mathcal{Y}\twoheadrightarrow \Sigma^*, \Enc_{\mathcal{Z}}:\mathcal{Z}\twoheadrightarrow \Sigma^*$ that are injective and whose images are computable.
For a partial computable function $f:\subseteq\mathcal{Z}\computableto\mathcal{Y}$, define the plain Kolmogorov complexity $C_U(f)$ of $f$ by
\begin{equation}
C_U(f) :=\min \Big\{|\bm{p}| \ \Big| \ \bm{p}\in\Programs_{U} (f)\Big\}.
\end{equation}
Here, $\Programs_{U} (f) \subseteq \Sigma^*$ is defined by
\begin{equation}
\begin{split}
& \bm{p} \in \Programs_{U} (f) \\
& \Leftrightarrow
\begin{cases}
U \Big(\overline{\Enc_{\mathcal{Z}} (z)} \cdot \bm{p}\Big) \downarrow = \Enc_{\mathcal{Y}} \big(f(z)\big) & \text{if $z \in \dom f$}, \\
U \Big(\overline{\Enc_{\mathcal{Z}} (z)} \cdot \bm{p}\Big) \uparrow & \text{if $z \in \mathcal{Z} \setminus \dom f$}. \\
\end{cases}
\end{split}
\end{equation}

For $x \in \mathcal{X}$, define the conditional plain Kolmogorov complexity $C_U (f \mid x)$ by
\begin{equation}
C_U (f \mid x)=\min \Big\{|\bm{p}| \ \Big| \ \bm{p} \in \Programs_{U} (f \mid x) \Big\}.
\end{equation}
Here, $\Programs_{U} (f \mid x) \subseteq \Sigma^*$ is defined by
\begin{equation}
\begin{split}
& \bm{p} \in \Programs_{U} (f \mid x) \\
& \Leftrightarrow
\begin{cases}
U \Big(\overline{\Enc_{\mathcal{Z}} (z)} \cdot \overline{\Enc_{\mathcal{X}} (x)} \cdot \bm{p}\Big) \downarrow = \Enc_{\mathcal{Y}} \big(f(z)\big) & \text{if $z \in \dom f$}, \\
U \Big(\overline{\Enc_{\mathcal{Z}} (z)} \cdot \overline{\Enc_{\mathcal{X}} (x)} \cdot \bm{p}\Big) \uparrow & \text{if $z \in \mathcal{Z} \setminus \dom f$}. \\
\end{cases}
\end{split}
\end{equation}
\end{definition}

In what follows, under a fixed injective encoding function $\Enc_{\QQ}: \QQ \to \Sigma^*$ for rational numbers whose image is computable, or decidable, for every string $\bm{a} \in \Sigma^*$, real number $r \in \RR$, and natural number $k \in \NN$, define $\bm{a} \approx_{k} r$ by $\bm{a} \approx_{k} r \Leftrightarrow \exists \tilde{r} \in \QQ, \Enc_{\QQ} (\tilde{r}) = \bm{a} \text{ and } |r - \tilde{r}| \le |\Sigma|^{-k}$.
This means that the rational number represented by $\bm{a}$ is an approximation to the real number $r$ with accuracy $|\Sigma|^{-k}$.

\begin{definition}[Definition of Kolmogorov complexity for real-valued functions]
\label{def:RealKolmogorov}
Let $f:\subseteq\mathcal{X}\computableto\RR$.
Define the set of programs $\Programs_{U} (f) \subseteq \Sigma^*$ that implement arbitrary-precision approximations of $f$ by
\begin{equation}
\begin{split}
& \bm{p} \in \Programs_{U} (f) \\
& \Leftrightarrow
\begin{cases}
U \Big(\overline{\Enc_{\mathcal{Z}} (z)} \cdot \overline{\Enc_{\mathcal{\NN}} (k)} \cdot \bm{p}\Big) \downarrow \approx_k f(z) & \text{for all $k \in \NN$, if $z \in \dom f$}, \\
U \Big(\overline{\Enc_{\mathcal{Z}} (z)} \cdot \overline{\Enc_{\mathcal{\NN}} (k)} \cdot \bm{p}\Big) \uparrow & \text{for all $k \in \NN$, if $z \in \mathcal{Z} \setminus \dom f$}. \\
\end{cases}
\end{split}
\end{equation}
Here, $\Enc_{\NN}$ and $\Enc_{\QQ}$ are fixed computable injective encodings of $\NN$ and $\QQ$, respectively.
Then
\begin{equation}
C_U(f):=\min\{|\bm p|\mid \bm p\in\Programs_U(f)\}
\end{equation}
is called the plain Kolmogorov complexity of $f$.
For $z\in\mathcal Z$ and a fixed encoding function
$\Enc_{\mathcal Z}:\mathcal Z\computableinjection\Sigma^*$,
define $\Programs_{U} (f \mid x) \subseteq \Sigma^*$ by
\begin{equation}
\begin{split}
& \bm{p} \in \Programs_{U} (f \mid x) \\
& \Leftrightarrow
\begin{cases}
U \Big(\overline{\Enc_{\mathcal{Z}} (z)} \cdot \overline{\Enc_{\mathcal{\NN}} (k)} \cdot \overline{\Enc_{\mathcal{X}} (x)} \cdot \bm{p}\Big) \downarrow \approx_k f(z) & \text{for all $k \in \NN$, if $z \in \dom f$}, \\
U \Big(\overline{\Enc_{\mathcal{Z}} (z)} \cdot \overline{\Enc_{\mathcal{\NN}} (k)} \cdot \overline{\Enc_{\mathcal{X}} (x)} \cdot \bm{p}\Big) \uparrow & \text{for all $k \in \NN$, if $z \in \mathcal{Z} \setminus \dom f$}. \\
\end{cases}
\end{split}
\end{equation}
and define
\begin{equation}
C_U (f \mid x):=\min \Big\{|\bm{p}| \ \Big| \ \bm{p} \in \Programs_{U} (f \mid x) \Big\}.
\end{equation}
This is called the plain conditional Kolmogorov complexity of $f$ conditioned on $x$.
\end{definition}

\subsection{Prefix-free Kolmogorov complexity}
The definition of plain Kolmogorov complexity is straightforward, but its properties are not convenient for computation.
An operation that receives two programs $\bm{p}_1, \bm{p}_2$ operating on a universal partial computable function and uses them as internal functions of another function will frequently appear later. Since $\bm{p}_1$ and $\bm{p}_2$ are not prefix-free, if they are simply concatenated and passed to another function as $\bm{p}_1 \cdot \bm{p}_2$, the function receiving them cannot correctly decompose them into $\bm{p}_1$ and $\bm{p}_2$. Therefore, one must always make one of them self-delimiting before passing them, as in $\overline{\bm{p}_1} \cdot \bm{p}_2$ or $\overline{\bm{p}_2} \cdot \bm{p}_1$.
This increases string length, and hence formulas for evaluating the plain Kolmogorov complexity of composed functions tend to become complicated.

The above problem does not arise if it is known that $\bm{p}_1$ and $\bm{p}_2$ each belong to some computable, or decidable, prefix-free set.
This is because, when $\bm{p}_1 \cdot \bm{p}_2$ is passed to another function, the receiving function can use prefix-freeness to computably decompose $\bm{p}_1 \cdot \bm{p}_2$ into $\bm{p}_1$ and $\bm{p}_2$ internally; more formally, a well-defined projection function exists and is computable.
For this reason, in order to make it easy to evaluate the quantities related to Kolmogorov complexity, it is convenient to use prefix-free Kolmogorov complexity, which restricts the domain to prefix-free sets.
We now define prefix-free Kolmogorov complexity.

\begin{definition}[Conditional prefix-free partial computable function]
A partial computable function $f: \subseteq \Sigma^* \computableto \Sigma^*$ is a \NewTerm{conditional prefix-free partial computable function} if, for every $\bm{x} \in \Sigma^*$, the set $\dom f(\overline{\bm{x}} \cdot \bullet) := \{\bm{w} \in \Sigma^* \mid f(\overline{\bm{x}} \cdot \bm{w}) \downarrow \}$ is prefix-free.
\end{definition}

\begin{definition}[Universal conditional prefix-free function]
A conditional prefix-free partial computable function $V: \subseteq \Sigma^* \computableto \Sigma^*$ is a \NewTerm{universal conditional prefix-free function} if there exists a prefix-free computable set $\mathsf{PFPrograms}_{V}$ such that, for every conditional prefix-free partial computable function $f$, there exists $\bm{p} \in \mathsf{PFPrograms}_{V}$ for which the equality of partial functions $V (\overline{\bullet} \cdot \bm{p} \cdot \diamond) = f(\overline{\bullet} \cdot \diamond)$ holds, including their domains. More concretely, for every $\bm{x}, \bm{w} \in \Sigma^*$, the following holds:
\begin{equation}
\begin{cases}
V(\overline{\bm{x}} \cdot \bm{p} \cdot \bm{w}) \downarrow = f(\overline{\bm{x}} \cdot \bm{w}) \downarrow & \quad \text{if $\overline{\bm{x}} \cdot \bm{w} \in \dom f$}, \\
V(\overline{\bm{x}} \cdot \bm{p} \cdot \bm{w}) \uparrow & \quad \text{if $\overline{\bm{x}} \cdot \bm{w} \notin \dom f$}.
\end{cases}    
\end{equation}

\end{definition}

\begin{definition}[Prefix-free Kolmogorov complexity]
Fix one universal conditional prefix-free function $V: \subseteq \Sigma^* \computableto \Sigma^*$.
For every $\bm{x}, \bm{y} \in \Sigma^*$, define the \NewTerm{conditional prefix-free Kolmogorov complexity} $K_V (\bm{y} \mid \bm{x}) \in \NN$ by
\begin{equation}
K_V (\bm{y} \mid \bm{x}) := \min \Big\{|\bm{p}| \ \Big| \ \bm{p} \in \Sigma^*, V(\overline{\bm{x}} \cdot \bm{p}) \downarrow = \bm{y} \Big\}.
\end{equation}
Moreover, for every $\bm{y} \in \Sigma^*$, define the \NewTerm{prefix-free Kolmogorov complexity} $K_V (\bm{y}) \in \NN$ by conditioning on the empty string $()$, that is, $K_V (\bm{y}) := K_V \big(\bm{y} \ \big| \ () \big)$.
\end{definition}

An invariance theorem also holds for prefix-free Kolmogorov complexity.

\begin{proposition}[Invariance theorem for prefix-free Kolmogorov complexity]
\label{prp:PrefixInvariance}
Consider two universal conditional prefix-free functions $V, V': \subseteq \Sigma^* \computableto \Sigma^*$.
There exists a constant $c_{V, V'} \in \NN$ such that, for every $\bm{x}, \bm{y} \in \Sigma^*$, $K_{V'} (\bm{y} \mid \bm{x}) \le K_{V} (\bm{y} \mid \bm{x}) + c_{V', V}$.
Note that $c_{V', V}$ does not depend on $\bm{x}, \bm{y}$.
\end{proposition}

\begin{proof}
This is shown in the same way as Proposition \ref{prp:PlainInvariance}.
\end{proof}

\begin{definition}[Prefix-free Kolmogorov complexity of general elements]
Let $\mathcal{X}, \mathcal{Y}$ be at most countable sets, and fix their encoding functions $\Enc_{\mathcal{X}}: \mathcal{X} \to \Sigma^*$ and $\Enc_{\mathcal{Y}}: \mathcal{Y} \to \Sigma^*$. Both are assumed to be injective and to have computable, or decidable, images.  
Also fix one universal conditional prefix-free partial computable function $V: \subseteq \Sigma^* \computableto \Sigma^*$.
For every $x \in \mathcal{X}$ and $y \in \mathcal{Y}$, define the \NewTerm{conditional prefix-free Kolmogorov complexity} $K_V (y \mid x) \in \NN$ by $K_V (y \mid x) = K_V \big(\Enc_{\mathcal{Y}} (y) \ \big| \ \Enc_{\mathcal{X}} (x) \big)$. Moreover, for every $y \in \mathcal{Y}$, define the \NewTerm{prefix-free Kolmogorov complexity} $K_V (y) \in \NN$ by $K_V (\Enc_{\mathcal{Y}}(y))$.
\end{definition}

Next, we define the prefix-free Kolmogorov complexity of real-valued functions.
\begin{definition}[Prefix-free Kolmogorov complexity of real-valued functions]
Let $\mathcal{X}, \mathcal{Z}$ be at most countable sets, and fix their encoding functions $\Enc_{\mathcal{X}}: \mathcal{X} \to \Sigma^*$ and $\Enc_{\mathcal{Z}}: \mathcal{Z} \to \Sigma^*$. Also fix an encoding $\Enc_{\QQ}: \QQ \to \Sigma^*$ for rational numbers. All of these are assumed to be injective and to have computable, or decidable, images.
For a real-valued function $f: \subseteq \mathcal{Z} \to \RR$, define the set of programs $\mathsf{PFPrograms}_{V} (f \mid x) \subsetneq \Sigma^*$ that output arbitrary-precision approximations of $f$ conditioned on $x \in \mathcal{X}$ by
\begin{equation}
\begin{split}
& \bm{p} \in \mathsf{PFPrograms}_{V} (f \mid x)
\Leftrightarrow \\
& 
\begin{cases}
\forall k \in \NN, V (\overline{\Enc_{\mathcal{Z}} (z)} \cdot \overline{\Enc_{\NN}(k)} \cdot \overline{\Enc_{\mathcal{X}} (x)} \cdot \bm{p}) \downarrow \ \approx_{k} f(z) & \text{if $z \in \dom f$}, \\ \forall k \in \NN, V (\overline{\Enc_{\mathcal{Z}} (z)} \cdot \overline{\Enc_{\NN}(k)} \cdot \overline{\Enc_{\mathcal{X}} (x)} \cdot \bm{p}) \uparrow &\text{otherwise.} \end{cases}
\end{split}
\end{equation}
Then define the conditional prefix-free Kolmogorov complexity $K_{V} (f \mid x)$ of $f$ conditioned on $x \in \mathcal{X}$ by $K_{V} (f \mid x) = \min \{|\bm{p}| \mid \bm{p} \in \mathsf{PFPrograms}_{V} (f \mid x)
\}$.

Moreover, for a real-valued function $f: \subseteq \mathcal{Z} \to \RR$, define the set of programs $\mathsf{PFPrograms}_{V} (f) \subsetneq \Sigma^*$ that output arbitrary-precision approximations of $f$ by
\begin{equation}
\begin{split}
& \bm{p} \in \mathsf{PFPrograms}_{V} (f)
\Leftrightarrow \\
& 
\begin{cases}
\forall k \in \NN, V (\overline{\Enc_{\mathcal{Z}} (z)} \cdot \overline{\Enc_{\NN}(k)} \cdot \bm{p}) \downarrow \ \approx_{k} f(z) & \text{if $z \in \dom f$}, \\ \forall k \in \NN, V (\overline{\Enc_{\mathcal{Z}} (z)} \cdot \overline{\Enc_{\NN}(k)} \cdot \bm{p}) \uparrow &\text{otherwise.} \end{cases}
\end{split}
\end{equation}
Then define the prefix-free Kolmogorov complexity $K_V (f) \in \NN$ of $f$ by $K_V (f) = \min \{|\bm{p}| \mid \bm{p} \in \mathsf{PFPrograms}_{V} (f)
\}$.
\end{definition}

\begin{proposition}[Relation between plain Kolmogorov complexity and prefix-free Kolmogorov complexity]
\label{prm:Plain2Prefix}
Let $U: \subseteq \Sigma^* \to \Sigma^*$ be a universal partial computable function, and let $V: \subseteq \Sigma^* \to \Sigma^*$ be a universal conditional prefix-free partial computable function.
For the input space, or source set, $\mathcal{X}$ and the output space, or target set, $\mathcal{Y}$, fix encoding functions $\Enc_{\mathcal{X}}: \mathcal{X} \twoheadrightarrow \Sigma^*$ and $\Enc_{\mathcal{Y}}: \mathcal{Y} \twoheadrightarrow \Sigma^*$ that are injective and whose images are computable.
Then there exist constants $c_{U, V}, c_{V, U} \in \NN$ such that, for every $x \in \mathcal{X}$ and $y \in \mathcal{Y}$, the following hold:
\begin{equation}
\label{eqn:CtoK}
C_U (y \mid x) \le K_V (y \mid x) + c_{U,V}.    
\end{equation}

\begin{equation}
\label{eqn:KtoC}
K_V (y \mid x) \le C_U (y \mid x) + 2 \log_{|\Sigma|} \Big(C_U (y \mid x) + 1\Big) + c_{V,U}.    
\end{equation}
Note that $c_{U,V}, c_{V,U}$ do not depend on $\bm{x}, \bm{y}$. 
\end{proposition}

\begin{proof}
Even in the case of general $\mathcal{X}, \mathcal{Y}$, Kolmogorov complexity is ultimately defined on $\Sigma^*$, so it suffices to show the case $\mathcal{X} = \mathcal{Y} = \Sigma^*$ and $\Enc_{\mathcal{X}} = \Enc_{\mathcal{Y}} = \mathrm{id}_{\Sigma^*}$.
In what follows, we consider only the case $x = \bm{x} \in \Sigma^*$ and $y = \bm{y} \in \Sigma^*$.

\noindent \textbf{Proof of \eqref{eqn:CtoK}}
To upper-bound $C_U (\bm{y} \mid \bm{x})$, it suffices to construct $\bm{p}'$ such that $U(\overline{\bm{x}} \cdot \bm{p}') \downarrow = \bm{y}$.
By the definition of prefix-free Kolmogorov complexity, there exists $\bm{p}^* \in \Sigma^*$ satisfying $V (\overline{\bm{x}} \cdot \bm{p}^*) \downarrow = \bm{y}$ and $|\bm{p}^*| = K_V (\bm{y} \mid \bm{x})$.
By the partial computability of $V$ and the universality of $U$, there exists $\bm{p}_V \in \Sigma^*$ such that $U (\overline{\bm{x}} \cdot \bm{p}_V \cdot \bm{p}^*) = V (\overline{\bm{x}} \cdot \bm{p}^*) = \bm{y}$.
Therefore, by the definition of plain Kolmogorov complexity, the following holds:
\begin{equation}
C_U (\bm{y} \mid \bm{x}) \le |\bm{p}_V \cdot \bm{p}^*| = |\bm{p}_V| + |\bm{p}^*| = |\bm{p}_V| + K_V (\bm{y} \mid \bm{x}).
\end{equation}

Since $\bm{p}_V$ does not depend on $\bm{x}, \bm{y}$, the proof is complete by taking $c_{U,V} = |\bm{p}_V|$.

\noindent \textbf{Proof of \eqref{eqn:KtoC}}

To upper-bound $K_V (\bm{y} \mid \bm{x})$, it suffices to construct $\bm{p}'$ such that $V(\overline{\bm{x}} \cdot \bm{p}') \downarrow = \bm{y}$.
First, by the definition of plain Kolmogorov complexity, there exists $\bm{p}^* \in \Sigma^*$ satisfying $U (\overline{\bm{x}} \cdot \bm{p}^*) \downarrow = \bm{y}$ and $|\bm{p}^*| = C_U (y \mid x)$.
Fix two distinct elements belonging to $\Sigma$ and identify them with $0$ and $1$.
Fix one self-delimiting function $\overline{\bullet}^*: \Sigma^* \computableto \Sigma^*$ by $\overline{\bm{w}}^* := 1^{\left|\Enc^*_\NN {|\bm{w}|}\right|-1} \cdot 0 \cdot \Enc^*_\NN {|\bm{w}|} \cdot \bm{w}$. Here, $1^{\left|\Enc^*_\NN {|\bm{w}|}\right|-1}$ denotes the string obtained by repeating the character $1$ exactly $\Big|\Enc^*_\NN {|\bm{w}|}\Big|-1$ times. Note that $|\overline{\bm{w}}^*| = |\bm{w}| + 2 \lceil \log_{|\Sigma|} (\min \{|\bm{w}|, 1\}) \rceil \le |\bm{w}| + 2 \log_{|\Sigma|} (|\bm{w}|+1) + 1$.
Define a partial computable function $f: \subseteq \Sigma^* \computableto \Sigma^*$ whose domain is prefix-free as follows.
\begin{itemize}
\item \textbf{Input}: $\bm{z} \in \Sigma^*$.
\item \textbf{Step 1}: Find $\bm{x}, \bm{p} \in \Sigma^*$ such that $\bm{z} = \overline{\bm{x}}^* \cdot \overline{\bm{p}}^*$. By the computability and prefix-freeness of the image of $\overline{\bullet}^*$, if such $\bm{x}, \bm{p} \in \Sigma^*$ exist, then they are unique and can be found by computation. If it is determined by computation that no such $\bm{x}, \bm{p} \in \Sigma^*$ exist, intentionally enter an infinite loop; that is, set $f(\bm{z}) \uparrow$.
\item \textbf{Step 2}: Compute $U (\overline{\bm{x}} \cdot \bm{p})$; here we use the fact that $U$ is partial computable. If $U (\overline{\bm{x}} \cdot \bm{p}) \uparrow$, this corresponds to an infinite loop, and hence $f(\bm{z}) \uparrow$. If $U (\overline{\bm{p}} \cdot \bm{x}) \downarrow$, let $\bm{y} = U (\overline{\bm{x}} \cdot \bm{p})$.
\item \textbf{Step 3}: \textbf{Output} $\bm{y}$ then \textbf{terminate}.
\end{itemize}
Then, if $U(\overline{\bm{x}} \cdot \bm{p}) \downarrow$, we have $U(\overline{\bm{x}} \cdot \bm{p}) \downarrow = f(\overline{\bm{x}} \cdot \overline{\bm{p}}) \downarrow$.
By the universality of $V$, there exists $\bm{p}_f \in \mathsf{PFPrograms}_{V}$ such that, for every $\bm{x}, \bm{p} \in \Sigma^*$, $V(\overline{\bm{x}} \cdot \bm{p}_f \cdot \overline{\bm{p}}) = f(\overline{\bm{x}} \cdot \overline{\bm{p}})$.
Taking $\bm{p} = \bm{p}^*$, we have
$V(\overline{\bm{x}} \cdot \bm{p}_f \cdot \overline{\bm{p}^*}) = f(\overline{\bm{x}} \cdot \overline{\bm{p}^*}) = U (\overline{\bm{x}} \cdot \bm{p}^*) = \bm{y}$.
Therefore, by the definition of prefix-free Kolmogorov complexity, the following holds:
\begin{equation}
\begin{split}
K_V (\bm{y} \mid \bm{x}) &\le |\bm{p}_f \cdot \overline{\bm{p}^*}| \\
&\le |\bm{p}_f| + |\bm{p}^*| + 2 \log_{|\Sigma|} (|\bm{p}^*| + 1) + 1 \\
&= |\bm{p}_f| + C_U (\bm{y} \mid \bm{x}) + 2 \log_{|\Sigma|} (C_U (\bm{y} \mid \bm{x}) + 1) + 1.
\end{split}
\end{equation}
Since $\bm{p}_f$ does not depend on $\bm{x}, \bm{y}$, the proof is complete by taking $c_{V,U} = |\bm{p}_f| + 1$.
\end{proof}

\begin{proposition}
Let $U: \subseteq \Sigma^* \to \Sigma^*$ be a universal partial computable function, and let $V: \subseteq \Sigma^* \to \Sigma^*$ be a universal conditional prefix-free partial computable function.
Let $\mathcal{X}, \mathcal{Z}$ be at most countable sets, and fix their encoding functions $\Enc_{\mathcal{X}}: \mathcal{X} \to \Sigma^*$ and $\Enc_{\mathcal{Z}}: \mathcal{Z} \to \Sigma^*$. Also fix an encoding $\Enc_{\QQ}: \QQ \to \Sigma^*$ for rational numbers. All of these are assumed to be injective and to have computable, or decidable, images.
Then there exists a constant $c_{V, U}$ such that, for every real-valued function $f: \subseteq \mathcal{Z} \to \RR$ and every $x \in \mathcal{X}$, the following holds:
\begin{equation}
K_V (f \mid x) \le C_U (f \mid x) + 2 \log_{|\Sigma|} \Big(C_U (f \mid x) + 1\Big) + c_{V,U}.    
\end{equation}
\end{proposition}

\begin{proof}
This can be proved in the same way as Proposition \ref{prm:Plain2Prefix}.
\end{proof}

\newpage

\section{Proofs}
\subsection{Proof of Proposition \ref{prp:AIandComputable}}

\begin{proof}
Let $A=(f,\tau)$ be a stochastic input-output AI.
Fix $x\in\mathcal{X}'$ and $y\in\mathcal Y$.
By definition,
\begin{equation}
Q_A(y\mid x)
=
\sum_{\bm u\in\mathcal T_x}
n^{-|\bm u|}\mathbbm 1(f(\bm u,x)=y).
\end{equation}

First, we confirm that $Q_A(\cdot\mid x)$ is a probability mass function.
Since $\mathcal T_x$ is prefix-free, Kraft's inequality gives
\begin{equation}
\sum_{\bm u\in\mathcal T_x}n^{-|\bm u|}\le 1.
\end{equation}
On the other hand, the constructive stopping guarantee gives, for every $k\in\NN$,
\begin{equation}
1-
\sum_{\substack{\bm u\in\mathcal T_x\\ |\bm u|\le N_{\mathrm{stop}}(x,k)}}
n^{-|\bm u|}
\le
|\Sigma|^{-k}.
\end{equation}
Therefore, by letting $k\to\infty$, we obtain
\begin{equation}
\label{eqn:AIStop}
\sum_{\bm u\in\mathcal T_x}n^{-|\bm u|}=1.
\end{equation}
Note that the monotone convergence theorem can be used to justify this limiting operation.
Thus, the intended operation halts with probability $1$, and $Q_A(\cdot\mid x)$ is a probability mass function rather than a sub-probability mass function.

Next, we show that $Q_A(y\mid x)$ is computable.
For $N\in\NN$, define the finite partial sum
\begin{equation}
S_N(y\mid x)
:=
\sum_{\substack{\bm u\in\mathcal T_x\\ |\bm u|\le N}}
n^{-|\bm u|}\mathbbm 1(f(\bm u,x)=y).
\end{equation}
Since the set $\{\bm u\in(\NN_{<n})^*\mid |\bm u|\le N\}$ is finite, the above sum is a finite sum.
Furthermore, since $\tau$ is total computable on $(\NN_{<n})^*\times\mathcal{X}'$, membership $\bm u\in\mathcal T_x$ can be decidably computed.
Moreover, if $\bm u\in\mathcal T_x$, then by the consistency between the domain of the main function and the accepted random-sequence set in Definition \ref{def:AI},
$(\bm u,x)\in\dom f$.
Therefore, $f(\bm u,x)$ can be computed in finite time.
Hence, $S_N(y\mid x)$ is a rational number computable from $x,y,N$.

For every $N$,
\begin{align}
0
&\le
Q_A(y\mid x)-S_N(y\mid x)\\
&=
\sum_{\substack{\bm u\in\mathcal T_x\\ |\bm u|>N}}
n^{-|\bm u|}\mathbbm 1(f(\bm u,x)=y)\\
&\le
\sum_{\substack{\bm u\in\mathcal T_x\\ |\bm u|>N}}
n^{-|\bm u|}\\
&=
1-
\sum_{\substack{\bm u\in\mathcal T_x\\ |\bm u|\le N}}
n^{-|\bm u|}.
\end{align}
Here, the last equality follows from \eqref{eqn:AIStop}.
For any accuracy $k\in\NN$, by the construction of $N_{\mathrm{stop}}$,
\begin{equation}
1-
\sum_{\substack{\bm u\in\mathcal T_x\\ |\bm u|\le N_{\mathrm{stop}}(x,k)}}
n^{-|\bm u|}
\le
|\Sigma|^{-k}.
\end{equation}
Therefore,
\begin{equation}
0\le
Q_A(y\mid x)-S_{N_{\mathrm{stop}}}(x,k)(y\mid x)
\le
|\Sigma|^{-k}.
\end{equation}
Thus, by computing
$S_{N_{\mathrm{stop}}(x,k)}(y\mid x)$
from $(k,x,y)$, one can approximate $Q_A(y\mid x)$ within error $|\Sigma|^{-k}$.
Therefore, $Q_A(y\mid x)$ is computable to arbitrary precision from $x,y$.
That is,
\begin{equation}
Q_A(\cdot\mid\cdot):\mathcal{X}'\times\mathcal Y\computableto[0,1]
\end{equation}
is a computable conditional probability mass function.
\end{proof}

\begin{remark}[Constructive stopping guarantee in LLMs]
In practical LLMs, a maximum generation length is usually set.
In this case, there exists a finite upper bound on random-number consumption that is computable from the input, and therefore the constructive stopping guarantee automatically holds.
\end{remark}

\subsection{Proof of Theorem \ref{thm:Quadrilemma}}

\begin{lemma}[Shannon--Fano--Elias coding]
\label{lem:ShannonFanoConstructive}
There exist conditional prefix-free computable functions $\mathsf{SFE}_{\mathsf{Enc}}: \subseteq \Sigma^* \computableto \Sigma^*$ and $\mathsf{SFE}_{\mathsf{Dec}}: \subseteq \Sigma^* \computableto \Sigma^*$ satisfying the following.

\noindent \textbf{Conditions satisfied by} $\mathsf{SFE}_{\mathsf{Enc}}: \subseteq \Sigma^* \computableto \Sigma^*$: For every computable conditional probability mass function $Q: \mathcal{X} \times \mathcal{Y} \computableto [0, 1]$, every $x \in \mathcal{X}$, and every $\bm{q} \in \mathsf{PFPrograms}_{V} \Big(Q(\bullet \mid x) \ \Big| \ x\Big)$, the following hold.

\begin{itemize}
\item For every $y \in \mathcal{Y}$ satisfying $Q (y \mid x) > 0$, $\mathsf{SFE}_{\mathsf{Enc}} \Big(\overline{\Enc_\mathcal{X} (x)} \cdot \overline{\Enc_\mathcal{Y} (y)} \cdot \bm{q} \Big) \downarrow$.
In this case, $\bm{c}_{y \mid x, \bm{q}} := \mathsf{SFE}_{\mathsf{Enc}} \Big(\overline{\Enc_\mathcal{X} (x)} \cdot \overline{\Enc_\mathcal{Y} (y)} \cdot \bm{q} \Big)$ can be regarded as a new code for $y$.
\item For every $y \in \mathcal{Y}$ satisfying $Q (y \mid x) > 0$, the code length satisfies the inequality $|\bm{c}_{y \mid x, \bm{q}}| \le - \log_{|\Sigma|} Q (y \mid x) + 4$.
\item The code set $\mathcal{C}_{x, \bm{q}} := \Big\{\bm{c}_{y \mid x, \bm{q}} \ \Big| \ y \in \mathcal{Y}, Q (y \mid x) > 0 \Big\}$ is prefix-free.
\end{itemize}

\noindent \textbf{Conditions satisfied by} $\mathsf{SFE}_{\mathsf{Dec}}: \subseteq \Sigma^* \computableto \Sigma^*$: For every computable conditional probability mass function $Q: \mathcal{X} \times \mathcal{Y} \computableto [0, 1]$, every $x \in \mathcal{X}$, and every $\bm{q} \in \mathsf{PFPrograms}_{V} \Big(Q(\bullet \mid x) \ \Big| \ x\Big)$, the following holds.
\begin{itemize}
\item For every $y \in \mathcal{Y}$ satisfying $Q (y \mid x) > 0$, $\mathsf{SFE}_{\mathsf{Dec}} \Big(\overline{\Enc_{\mathcal{X}} (x)} \cdot \bm{q} \cdot \bm{c}_{y \mid x, \bm{q}}\Big) \downarrow = \Enc_{\mathcal{Y}} (y)$.
\end{itemize}

In this paper, the above coding scheme is called \NewTerm{Shannon-Fano-Elias coding}. $\mathsf{SFE}_{\mathsf{Enc}}$ and $\mathsf{SFE}_{\mathsf{Dec}}$ are the encoder and decoder in Shannon-Fano-Elias coding, respectively.
\end{lemma}

\begin{proof}
The idea of Shannon-Fano-Elias coding is as follows.
Using the fact that the image of $\Enc_{\mathcal{Y}}$ is computable, or decidable, order the elements of $\mathcal{Y}$ lexicographically and, based on this order, use $\bm{q}$ to construct the cumulative distribution function $F: \NN \to [0, 1]$ with sufficient precision.
For each element $y \in \mathcal{Y}$, take the fractional representation of a finite base-$|\Sigma|$ decimal $c_{i} = 0.c_{i,1} c_{i,2} \cdots c_{i,N_i}$ that lies inside the cumulative-distribution interval $\Big[F (i), F (i + 1)\Big)$ corresponding to its lexicographic index $i$, and output it as the code $\bm{c}_{y \mid x, \bm{q}}$ for $y$.
At this time, choose it so that the interval $\Big[c_{i}, \overline{c}_{i}\Big)$ of real numbers truncated to the finite decimal $c_{i}$, where $\overline{c}_{i} := c_i + |\Sigma|^{-k_i}$, is always completely contained in the cumulative-distribution interval corresponding to $y$, that is, so that $\Big[c_{i}, \overline{c}_{i}\Big) \subset \Big[F (i), F (i + 1)\Big)$.
This guarantees the prefix-freeness of the code set $\mathcal{C}_{x, \bm{q}} := \Big\{\bm{c}_{y \mid x, \bm{q}} \ \Big| \ y \in \mathcal{Y} \Big\}$.
Indeed, for $i \ne i'$, the property of the cumulative distribution function implies $\Big[F (i), F (i + 1)\Big) \cap \Big[F (i'), F (i' + 1)\Big) = \{\}$, and from this it follows that $\Big[c_{i}, \overline{c}_{i}\Big) \cap \Big[c_{i'}, \overline{c}_{i'}\Big) = \{\}$. This clearly implies that neither code is a prefix of the other.
Therefore, what must be done is to choose $c_i$ concretely, with as few digits as possible, and so that $\Big[c_{i}, \overline{c}_{i}\Big) \subset \Big[F (i), F (i + 1)\Big)$.
This can be achieved by the following algorithm using $\bm{q}$, which approximates $Q$ to arbitrary precision.

\noindent \textbf{Computation algorithm for} $\mathsf{SFE}_{\mathsf{Enc}}$

\begin{itemize}
\item \textbf{Input}: $\overline{\Enc_{\mathcal{X}} (x)} \cdot \overline{\Enc_{\mathcal{Y}} (y)} \cdot \bm{q}$.
\item \textbf{Step 1}: Using prefix-freeness, separate the input into $\overline{\Enc_{\mathcal{X}} (x)} \in \overline{\Enc_{\mathcal{X}} (\mathcal{X})}$, $\overline{\Enc_{\mathcal{Y}} (y)} \in \overline{\Enc_{\mathcal{Y}} (\mathcal{Y})}$, and $\bm{q} \in \mathsf{PFPrograms}_{V}$. If they cannot be separated correctly, enter an infinite loop intentionally.
\item \textbf{Step 2}: Using the computability, or decidability, of the image of $\Enc_{\mathcal{Y}}$, obtain $\Enc_{\mathcal{Y}} (y)$ from $\overline{\Enc_{\mathcal{Y}} (y)}$.
\item \textbf{Step 3}: By scanning the elements of $\Sigma^*$ in lexicographic order, obtain the first $i$ elements $\bm{y}_0, \bm{y}_1, ..., \bm{y}_i$ in the lexicographic ordering of the set $\Enc_{\mathcal{Y}} (\mathcal{Y}) \subseteq \Sigma^*$. Here, $i$ is determined by $\bm{y}_i = \Enc_{\mathcal{Y}} (y)$. In what follows, define $y_j \in \mathcal{Y}$ by $\Enc_{\mathcal{Y}} (y_j) = \bm{y}_j$ for $j = 0, 1, ..., i$. 
\item \textbf{Step 4}: Using $\bm{q}$, iterate $k = 0, 1, ...$ in ascending order, compute an approximation $\tilde{Q}_i^{(k)}$ of $Q (y_i \mid x)$ with accuracy $\pm |\Sigma|^{-k}$, and find $k^* = \min \Big\{ k \in \NN \ \Big| \ \tilde{Q}_i^{(k)} \ge |\Sigma|^{-(k-3)} - |\Sigma|^{-k}\Big\}$. Such a $k^*$ necessarily exists, because $\tilde{Q}_i^{(k)} \to Q (y_i \mid x) > 0$ and $|\Sigma|^{-(k-3)} - |\Sigma|^{-k} \to 0$ as $k \to +\infty$.
\item \textbf{Step 5}: Using $\bm{q}$, compute an approximation $\tilde{F}_i^{(k^*)}$ of $F_{i} := \sum_{j=0}^{i - 1} Q (y_j \mid x)$ with accuracy $\pm |\Sigma|^{-k^*}$.
\item \textbf{Step 6}: Let $c_i$ be obtained by rounding up $\tilde{F}_i^{(k^*)} + 2|\Sigma|^{-k^*}$ after the $(k^*+1)$-st digit after the decimal point, and let the first $k$ digits after the decimal point, $\bm{c}_{Q, y \mid x} = c_{i,1} c_{i,2} \cdots c_{i,k^*} \in \Sigma^{k^*}$, be the code for $y$. In formulas, let $c_i = \min_{\bm{c}_{i} \in \Sigma^*} \Big\{c \ \Big| \ c = \sum_{\ell=1}^{k^*} c_{i, \ell} |\Sigma|^{-\ell}, \ c \ge \tilde{F}_i^{(k^*)} + 2 |\Sigma|^{-k^*}\Big\}$, and let $\bm{c}_i = \arg \min_{\bm{c}_{i} \in \Sigma^*} \Big\{c \ \Big| \ c = \sum_{\ell=1}^{k^*} c_{i, \ell} |\Sigma|^{-\ell}, \ c \ge \tilde{F}_i^{(k^*)} + 2 |\Sigma|^{-k^*}\Big\}$.
\item \textbf{Step 7}: \textbf{Output} $\bm{c}_{y \mid x, \bm{q}} = \bm{c}_i$.
\end{itemize}
We prove that $\mathsf{SFE}_{\mathsf{Enc}}$ defined above satisfies the desired properties.
The prefix-freeness of the domain of $\mathsf{SFE}_{\mathsf{Enc}}$ is guaranteed by \textbf{Step 1}.
What remains to prove is the code-length condition $k^* < - \log_{|\Sigma|} Q (y \mid x) + 4$, and the sufficient condition for prefix-freeness of the code set, namely $\Big[c_i, c_i + |\Sigma|^{-k^*}\Big) \subseteq \Big[F_i, F_{i+1}\Big)$.
By the definition of $k^*$,
\begin{equation}
\label{eqn:KStar}
\tilde{Q}_i^{(k^*)} \ge |\Sigma|^{-(k^*-3)} - |\Sigma|^{-k^*}    
\end{equation}
and 
\begin{equation}
\label{eqn:KStar-1}
\tilde{Q}_i^{(k^*-1)} < |\Sigma|^{-(k^*-4)} - |\Sigma|^{-(k^*-1)}    
\end{equation}
hold.
First, by Equation \eqref{eqn:KStar-1} and the fact that $\tilde{Q}_i^{(k^*-1)}$ is a $\pm |\Sigma|^{-(k^*-1)}$ approximation,
\begin{equation}
Q (y \mid x) \le \tilde{Q}_i^{(k^*-1)} + |\Sigma|^{-(k^*-1)} < \Big(|\Sigma|^{-(k^*-4)} - |\Sigma|^{-(k^*-1)}\Big) + |\Sigma|^{-(k^*-1)} \le |\Sigma|^{-(k^*-4)}.
\end{equation}
Therefore, $|\bm{c}_{i}| = k^* < - \log_{|\Sigma|} Q (y \mid x) + 4$ holds.

It remains to prove the sufficient condition for prefix-freeness, namely $\Big[c_i, c_i + |\Sigma|^{-k^*}\Big) \subseteq \Big[F_i, F_{i+1}\Big)$.
Here, $F_i = \sum_{j=0}^{i-1} Q (y_j \mid x)$ and $F_{i+1} = F_i + Q (y_i \mid x)$.
In what follows, we prove the stronger statement
\begin{equation}
\label{eqn:SFEMargin}    
\Big[c_i, c_i + |\Sigma|^{-k^*}\Big) \subseteq \Big[F_i + |\Sigma|^{-k^*}, F_{i+1} - |\Sigma|^{-k^*}\Big).
\end{equation}
For this purpose, it suffices to prove $c_i \ge F_i + |\Sigma|^{-k^*}$ and $\overline{c}_i := c_i + |\Sigma|^{- k^*} \le F_{i+1} - |\Sigma|^{- k^*}$.
First, by the construction of $c_i$, $c_i \ge \tilde{F}_{i}^{(k^*)} + 2|\Sigma|^{-{k^*}}$, and by the fact that $\tilde{F}_{i}^{(k^*)}$ is a $\pm |\Sigma|^{-{k^*}}$ approximation, it follows that $c_i \ge F_i + |\Sigma|^{-{k^*}}$.
Similarly, since the increase in the rounding-up operation in the construction of $c_i$ is less than $|\Sigma|^{-k^*}$, we have $c_i < \tilde{F}_{i}^{(k^*)} + 2|\Sigma|^{-k^*} + |\Sigma|^{-k^*} = \tilde{F}_{i}^{(k^*)} + 3|\Sigma|^{-k^*}$.
Furthermore, by the fact that $\tilde{F}_{i}^{(k^*)}$ is a $\pm |\Sigma|^{-{k^*}}$ approximation, the following can be said:
\begin{equation}
\begin{split}
\overline{c}_i := c_i + |\Sigma|^{- k^*} &< \Big(\tilde{F}_{i}^{(k^*)} + 3 |\Sigma|^{-k^*}\Big) + |\Sigma|^{- k^*} \\
&\le \Big(\big(F_i + |\Sigma|^{-k^*}\big) + 3 |\Sigma|^{-k^*}\Big) + |\Sigma|^{- k^*} \\
&\le F_i + \big(\tilde{Q}_i^{(k^*)} - |\Sigma|^{- k^*}\big)  - |\Sigma|^{- k^*} \\
&\le F_i + Q (y_i \mid x) - |\Sigma|^{- k^*}.
\end{split}
\end{equation}
Here, the third inequality can be proved from Equation \eqref{eqn:KStar-1} as follows:
\begin{equation}
\begin{split}
\big(\tilde{Q}_i^{(k^*)} - |\Sigma|^{- k^*}\big)  - |\Sigma|^{- k^*} 
&\ge |\Sigma|^{- (k^*-3)} - 3 |\Sigma|^{-k^*} \\
&\ge 8 |\Sigma|^{-k^*} - 3 |\Sigma|^{-k^*} \\
&= 5 |\Sigma|^{-k^*}.
\end{split}
\end{equation}
This proves the properties of $\mathsf{SFE}_{\mathsf{Enc}}$.

Next, we construct $\mathsf{SFE}_{\mathsf{Dec}}$ and prove that it has the desired properties.
The idea is simply to recover the $y_i$ corresponding to the index $i$ such that $c \in [F_i, F_{i+1})$, given the code $c$.
This operation should be carried out while taking approximation errors into account.

\noindent \textbf{Computation algorithm for} $\mathsf{SFE}_{\mathsf{Dec}}$

\begin{itemize}
\item \textbf{Input}: $\overline{\Enc_{\mathcal{X}} (x)} \cdot \bm{q} \cdot \bm{c}$
\item \textbf{Step 1}: Using prefix-freeness, separate the input into $\overline{\Enc_{\mathcal{X}} (x)} \in \overline{\Enc_{\mathcal{X}} (\mathcal{X})}$, $\bm{q} \in \mathsf{PFPrograms}_{V}$, and $\bm{c} \in \Sigma^*$. If they cannot be separated correctly, enter an infinite loop intentionally.
\item \textbf{Step 2}: Initialize $i=0$.
\item \textbf{Step 3}: By scanning $\Sigma^*$ in lexicographic order, find the element $\bm{y}_i$ that is the $i$-th element in lexicographic order among the elements of $\Enc_{\mathcal{Y}} (\mathcal{Y}) \subseteq \Sigma^*$.
\item \textbf{Step 4}:
Using $\bm{q}$, as in Step 4 of $\mathsf{SFE}_{\mathsf{Enc}}$, iterate $k = 0, 1, ...$ in ascending order, compute an approximation $\tilde{Q}_i^{(k)}$ of $Q (y_i \mid x)$ with accuracy $\pm |\Sigma|^{-k}$, and find $k^* = \min \Big\{ k \in \NN \ \Big| \ \tilde{Q}_i^{(k)} \ge |\Sigma|^{-(k-3)} - |\Sigma|^{-k}\Big\}$.
\item \textbf{Step 5}: Using $\bm{q}$, as in Step 5 of $\mathsf{SFE}_{\mathsf{Enc}}$, compute an approximation $\tilde{F}_i^{(k^*)}$ of $F_{i} := \sum_{j=0}^{i - 1} Q (y_j \mid x)$ with accuracy $\pm |\Sigma|^{-k^*}$.
\item \textbf{Step 6}: As in Step 6 of $\mathsf{SFE}_{\mathsf{Enc}}$, let $\bm{c}_i = \arg \min_{\bm{c}_{i} \in \Sigma^*} \Big\{c \ \Big| \ c = \sum_{\ell=1}^{k^*} c_{i, \ell} |\Sigma|^{-\ell}, \ c \ge \tilde{F}_i^{(k^*)} + 2 |\Sigma|^{-k^*}\Big\}$, and \textbf{If} $\bm{c}_i = \bm{c}$ \textbf{then go to Step 8.} 
\item \textbf{Step 7}: Set $i \gets i+1$, and \textbf{go to Step 3}.
\item \textbf{Step 8}: \textbf{Output} $\bm{y}_i$ \textbf{then terminate.}
\end{itemize}
It is clear from the construction that this is the inverse of $\mathsf{SFE}_{\mathsf{Enc}}$.
Furthermore, given $\overline{\Enc_{\mathcal{X}} (x)}$ and $\bm{q}$, the condition for halting is $\bm{c} \in \mathcal{C}_{x, \bm{q}}$, and $\mathcal{C}_{x, \bm{q}}$ is prefix-free. Hence $\dom \mathsf{SFE}_{\mathsf{Dec}}$ is also prefix-free.
This proves that $\mathsf{SFE}_{\mathsf{Dec}}$ has all the desired properties.
\end{proof}

\begin{remark}
Although it is not relevant to the subsequent proof, if \textbf{Step 5} and \textbf{Step 6} in the computation algorithm for $\mathsf{SFE}_{\mathsf{Dec}}$ are modified as follows, then at the cost of losing the prefix-freeness of the domain, one can use different elements of $\mathsf{PFPrograms}_{V} \Big(Q(\bullet \mid x) \ \Big| \ x\Big)$ at the time of encoding and decoding.
More specifically, for $\bm{q}, \bm{q}' \in \mathsf{PFPrograms}_{V} \Big(Q(\bullet \mid x) \ \Big| \ x\Big)$, for every $y \in  \mathcal{Y}$ satisfying $Q (y \mid x) > 0$, $\mathsf{SFE}_{\mathsf{Dec}} \Big(\overline{\Enc_{\mathcal{X}} (x)} \cdot \bm{q}' \cdot \bm{c}_{y \mid x, \bm{q}}\Big) \downarrow = \Enc_{\mathcal{Y}} (y)$ holds.

\begin{itemize}
\item \textbf{Step 5}: Using $\bm{p}$, compute approximations $\tilde{F}_i^{(k^*)}$ and $\tilde{F}_{i+1}^{(k^*)}$ with accuracy $\pm |\Sigma|^{-k^*}$ of $F_{i} := \sum_{j=0}^{i - 1} Q (y_j \mid x)$ and $F_{i} := \sum_{j=0}^{i - 1} Q (y_j \mid x)$, respectively.
\item \textbf{Step 6}: \textbf{If} $\sum_{\ell=1}^{k^*} c_{\ell} |\Sigma|^{-\ell} \in \Big[\tilde{F}_i^{(k^* + 1)} + |\Sigma|^{-(k^* + 1)}, \tilde{F}_{i+1}^{(k^* + 1)} - |\Sigma|^{-(k^* + 1)}\Big)$ \textbf{then go to Step 8.} 
\end{itemize}

This decodes correctly because, by Equation \eqref{eqn:SFEMargin}, $c_i \in \Big[F_i + |\Sigma|^{-k^*}, F_{i+1} - |\Sigma|^{-k^*}\Big)$, and
$\Big[F_i + |\Sigma|^{-k^*}, F_{i+1} - |\Sigma|^{-k^*}\Big) \subseteq \Big[\tilde{F}_i^{(k^* + 1)} + |\Sigma|^{-(k^* + 1)}, \tilde{F}_{i+1}^{(k^* + 1)} - |\Sigma|^{-(k^* + 1)}\Big)$ holds.
Therefore, \textbf{Step 6} necessarily stops at the correct $i$.
Moreover, since $\Big[\tilde{F}_i^{(k^* + 1)} + |\Sigma|^{-(k^* + 1)}, \tilde{F}_{i+1}^{(k^* + 1)} - |\Sigma|^{-(k^* + 1)}\Big) \subseteq [F_i, F_{i+1})$, the stopping conditions for different $i$ are mutually disjoint, and the algorithm never stops at an incorrect $i$.
\end{remark}

\begin{lemma}[Lower bound on log-likelihood by prefix-free complexity]
\label{lem:MainLemmaCorrected}
Let the base of $\log$ be $|\Sigma|$.
Let $V$ be a prefix-free universal function.
Let
$Q(\cdot\mid\cdot):\mathcal X\times\mathcal Y\computableto[0,1]$
be a constructively summable computable conditional probability mass function.
Then there exists a constant $c_{\mathsf{code}}$ such that, for all $Q$, all $x\in\mathcal X$, and all $y\in\mathcal Y$,
\begin{equation}
\label{eqn:MainPrefix}
-\log_{|\Sigma|} Q(y\mid x)+K_V(Q(\cdot\mid x)\mid x)+c_{\mathsf{code}}\ge K_V(y\mid x)
\end{equation}
holds.
Here, $c_{\mathsf{code}}$ arises from the additive constant for simulating the fixed Shannon--Fano--Elias decoder by $V$, and from the fixed additional length of the code construction, and does not depend on $Q,x,y$.
When $Q(y\mid x)=0$, the left-hand side is interpreted as $+\infty$.
Consequently, for every probability distribution $P\in\mathcal P(\mathcal X\times\mathcal Y)$,
\begin{equation}
\label{eqn:MainPrefixE}
\EE_{X,Y\sim P}[-\log Q(Y\mid X)]
+
\EE_{X,Y\sim P}K_V(Q(\cdot\mid X)\mid X)+c_{\mathsf{code}}
\ge
\EE_{X,Y\sim P}K_V(Y\mid X)
\end{equation}
holds as an inequality of extended real numbers.
\end{lemma}

\begin{proof}
If $Q(y\mid x)=0$, then the left-hand side is $+\infty$, and the claim is trivial.
In what follows, assume $Q(y\mid x)>0$.
What must be done is to find $\bm{p} \in \Sigma^*$ and $c$ such that $\bm{p}_{y \mid x}$ satisfies $V \Big(\overline{\Enc_{\mathcal{X}} (x)} \cdot\bm{p}_{y \mid x} \Big) \downarrow = \Enc_{\mathcal{Y}} (y)$ and $|\bm{p}_{y \mid x}| \le - \log Q(y \mid x) + K_V (Q (\cdot \mid \cdot) \mid x) + c$.

Lemma \ref{lem:ShannonFanoConstructive} asserts that there is a function $\mathsf{SFE}_ \mathsf{Dec}$, which is a prefix-free conditional computable function, with the following property:
\begin{itemize}
    \item For every computable conditional probability mass function $Q: \mathcal{X} \times \mathcal{Y} \computableto [0, 1]$, every $x \in \mathcal{X}$, $y \in \mathcal{Y}$, and every $\bm{q} \in \mathsf{PFPrograms}_{V} (Q (\bullet \mid x) \mid x)$, $\mathsf{SFE}_ \mathsf{Dec} \Big(\overline{\Enc_{\mathcal{X}} (x)} \cdot \bm{q} \cdot \bm{c}_{y \mid x, \bm{q}}\Big) \downarrow = \Enc_{\mathcal{Y}} (y)$ holds.
\end{itemize}

Therefore, by the universality of $V$, there exists a string $\bm{p}_{\mathsf{SFE}, \mathsf{Dec}} \in \Sigma^*$ satisfying the following:
\begin{itemize}
    \item For every computable conditional probability mass function $Q: \mathcal{X} \times \mathcal{Y} \computableto [0, 1]$, every $x \in \mathcal{X}$, $y \in \mathcal{Y}$, and every $\bm{q} \in \mathsf{PFPrograms}_{V} (Q (\bullet \mid x) \mid x)$, $V \Big(\overline{\Enc_{\mathcal{X}} (x)} \cdot \bm{p}_{\mathsf{SFE}, \mathsf{Dec}} \cdot \bm{q} \cdot \bm{c}_{y \mid x, \bm{q}}\Big) \downarrow = \Enc_{\mathcal{Y}} (y)$ holds.
\end{itemize}
Now take $\bm{q}$ to be some $\bm{q}^* \in \arg \min \mathsf{PFPrograms}_{V} \Big(Q (\bullet \mid x) \ \Big| \ x \Big)$.
By the definition of $K_V$, $|\bm{q}^*| = K_V \Big(Q(\bullet \mid x) \ \Big| \ x \Big)$.
Again, for every $x \in \mathcal{X}$ and $y \in \mathcal{Y}$, $V \Big(\overline{\Enc_\mathcal{X} (x)} \cdot \bm{p}_{\mathsf{SFE}, \mathsf{Dec}} \cdot \bm{q}^* \cdot \bm{c}_{y \mid x, \bm{q}^*} \Big) \downarrow = \Enc_{\mathcal{Y}} (y)$.
In other words, by defining $\bm{p}_{y \mid x} := \bm{p}_{\mathsf{SFE}, \mathsf{Dec}} \cdot \bm{q}^* \cdot \bm{c}_{y \mid x, \bm{q}^*}$, we can output $y$ in the form $V \Big(\overline{\Enc_\mathcal{X} (x)} \cdot \bm{p}_{y \mid x} \Big) \downarrow = \Enc_{\mathcal{Y}} (y)$.
Therefore, we obtain
\begin{align}
K_V (y \mid x) &\le \big|\bm{p}_{\mathsf{SFE}, \mathsf{Dec}} \cdot \bm{q}^{*} \cdot \bm{c}_{y \mid x, \bm{q}^{*}} \big| \\
&\le |\bm{p}_{\mathsf{SFE}, \mathsf{Dec}}| + K_V \Big(Q(\cdot \mid x) \ \Big| \ x \Big) - \log_{|\Sigma|} Q(y \mid x) + 4.
\end{align}
This proves Equation \eqref{eqn:MainPrefix} by setting $c_{\mathsf{code}} = |\bm{p}_{\mathsf{SFE}, \mathsf{Dec}}| + 4$, as it does not depend on $Q, x, y$.

The expectation version, Equation \eqref{eqn:MainPrefixE}, follows by integrating Equation \eqref{eqn:MainPrefix} with respect to $P$.
It also holds as an inequality of extended real numbers when the expectation contains $+\infty$.
\end{proof}

\begin{lemma}[Lemma for constructing a program for a function from a completely faithful explanation]
\label{lem:ExplanationToProgram}
Let $\Lambda$ be a finite character set, and fix an injective prefix-free encoding function for single characters
$\Enc^{\mathsf{PF}}_\Lambda:\Lambda \twoheadrightarrow \Sigma^*$.
Assume that the image $\Enc^{\mathsf{PF}}_\Lambda (\Lambda)$ is prefix-free and computable, or decidable.
Put
\begin{equation}
L_\Lambda:=\max \Big\{|\Enc^{\mathsf{PF}}_\Lambda(a)| \ \Big| \ a\in\Lambda \Big\}.
\end{equation}
For a string $\bm e=e_0\cdots e_{r-1}\in\Lambda^*$, define
\begin{equation}
\Enc_{\Lambda^*}(\bm e):=\Enc^{\mathsf{PF}}_\Lambda(e_0) \cdot \Enc^{\mathsf{PF}}_\Lambda(e_1) \cdot \cdots \cdot \Enc^{\mathsf{PF}}_\Lambda(e_{r-1}).
\end{equation}
Note that this is not prefix-free as an encoding function for variable-length $\Lambda^*$.
Then
\begin{equation}
|\Enc_{\Lambda^*}(\bm e)|\le L_\Lambda|\bm e|.
\end{equation}

Moreover, assume that $U$ is a universal conditional function and that an interpretation function $\Interpret$ is fixed.
Then there exists a constant $c_{\Interpret}$ such that, for any
$\bm{e}_x\in\Lambda^*$
that is a completely faithful explanation of
$Q (\bullet \mid x):\mathcal Y\computableto[0,1]$,
the following holds:
\begin{equation}
C_U (Q(\cdot\mid x)\mid x)
\le
L_\Lambda|\bm{e}_x|
+ c_{\Interpret}.
\end{equation}
\end{lemma}

\begin{proof}
The first inequality follows immediately from the fact that $\Enc_\Lambda(\bm e)$ is the concatenation of the encodings of each character, and that the code length per character is at most $L_\Lambda$.
We prove the final claim.
By the computability of $\Interpret$,
$\Interpret' (y, k, x, \bm{e}) = \Interpret (\bm{e}, k, y)$ is also computable.
There exists a partial computable function $\tilde{\Interpret}': \subseteq \Sigma^* \computableto \Sigma^*$ such that $\tilde{\Interpret}' \Big(\overline{\Enc_{\mathcal{Y}} (y)} \cdot \overline{\Enc_{\mathcal{\NN}} (k)} \cdot \overline{\Enc_{\mathcal{X}} (x)} \cdot \Enc_{\Lambda^*} (\bm{e}_{x}) \Big) \approx_k Q(y \mid x)$.

By the universality of $U$, there exists $\bm{p}_{\tilde{\Interpret}'} \in \Sigma^*$ such that
\begin{equation}
\begin{split}
& U \Big(\overline{\Enc_{\mathcal{Y}} (y)} \cdot \overline{\Enc_{\mathcal{\NN}} (k)} \cdot \overline{\Enc_{\mathcal{X}} (x)} \cdot \bm{p}_{\tilde{\Interpret}'} \cdot \Enc_{\Lambda^*} (\bm{e}_{x}) \Big) \\
& = \tilde{\Interpret}' \Big(\overline{\Enc_{\mathcal{Y}} (y)} \cdot \overline{\Enc_{\mathcal{\NN}} (k)} \cdot \overline{\Enc_{\mathcal{X}} (x)} \cdot \Enc_{\Lambda^*} (\bm{e}_{x}) \Big) \approx_k Q(y \mid x) 
\end{split}
\end{equation}
holds.
Therefore, by the definition of the plain Kolmogorov complexity of real-valued functions,
\begin{equation}
C_U \Big(Q(\bullet | x) \ \Big| \ x \Big) \le \Big|\bm{p}_{\tilde{\Interpret}'} \cdot \Enc_{\Lambda^*} (\bm{e}_{x})\Big| = |\bm{p}_{\Interpret}| + \Big|\Enc_{\Lambda^*} (\bm{e}_{x})\Big| \le |\bm{p}_{\tilde{\Interpret}'}| + L_{\Lambda} |\bm{e}_{x}|.
\end{equation}
Since $\bm{p}_{\tilde{\Interpret}'}$ does not depend on $Q$, $x$, or $\bm{e}_x$, the proof is complete by taking $c_{\Interpret} = |\bm{p}_{\tilde{\Interpret}'}|$.
\end{proof}

\begin{proof}[Proof of Theorem \ref{thm:Quadrilemma}]
We prove (i).
Combining the facts proved so far gives
\begin{align}
C_U(y\mid x)
&\le K_V(y\mid x)+c_{U,V}\\
&\le -\log Q_A(y\mid x)+K_V(Q_A(\cdot\mid x)\mid x)-c_{\mathsf{code}}+c_{U,V}\\
&\le -\log Q_A(y\mid x) + C_U (Q_A(\cdot\mid x)\mid x) + 2 \log_{|\Sigma|} \Big(C_U \big(Q_A(\cdot\mid x)\ \big| \ x\big) + 1 \Big) \\
& \quad + c_{V, U} + c_{\mathsf{code}}+c_{U,V}\\
&\le -\log Q_A(y\mid x)
+
L_\Lambda|\bm e_x| + c_{\Interpret}
+
2\log(L_\Lambda|\bm e_x|+1 + c_{\Interpret})\\
& \quad + c_{V, U} + c_{\mathsf{code}}+c_{U,V} \\
&\le -\log Q_A(y\mid x)
+
L_\Lambda|\bm e_x|
+
2\log(L_\Lambda|\bm e_x|+1 + c_{\Interpret})\\
& \quad + 3 c_{\Interpret} + 2 \log_{|\Sigma|} (c_{\Interpret} + 1) + c_{V, U} +c_{\mathsf{code}}+c_{U,V}.
\end{align}
Here, the first inequality uses \eqref{eqn:CtoK}, the second uses Lemma \ref{lem:MainLemmaCorrected}, the third uses \eqref{eqn:KtoC}, and the fourth uses Lemma \ref{lem:ExplanationToProgram}.
The last inequality applies the inequality
$\log_{b} (a + 1 + c) < \log_{b} (a + 1 + c + 1) \le \log_b (a + 1) + \log_b (c+1) + (c+1)$, which holds for arbitrary $a \ge 0, b \ge 2, c \ge 0$, with $a = L_\Lambda |\bm{e}_x|$, $b = |\Sigma|$, and $c = c_{\Interpret}$.
Therefore,
\begin{equation}
-\log Q_A(y\mid x)
+
L_\Lambda|\bm e_x|
+
2\left\lceil\log(L_\Lambda|\bm e_x|+1)\right\rceil
\ge
C_U(y\mid x)+c
\end{equation}
holds, where
\begin{equation}
c:= - 3 c_{\Interpret} - 2 \log_{|\Sigma|} (c_{\Interpret} + 1) - c_{V, U} - c_{\mathsf{code}} - c_{U,V}.
\end{equation}
This constant depends only on the fixed self-delimiting encoding, the fixed interpretation function, the fixed Shannon--Fano--Elias decoder, and the simulation constants between $U$ and $V$, and does not depend on $A,x,y,P$.

(ii) follows by integrating the pointwise inequality in (i) with respect to $P$.
It also holds as an inequality of extended real numbers when the expectation contains $+\infty$.
\end{proof}

\begin{remark}[On lower bounds for conditional Kolmogorov complexity]
In Lemma \ref{lem:MainLemmaCorrected}, under appropriate assumptions such as the computability of the distribution $P$, there is a classical lower bound for $\mathbb{E}_{X,Y\sim P}K_V(Y\mid X)$:
\begin{equation}
\mathbb{E}_{X,Y\sim P}K_V(Y\mid X)\ge H_{P}(Y\mid X)-O(1).
\end{equation}
However, if one simply replaces $\mathbb{E}_{X,Y\sim P}K_V(Y\mid X)$ in the inequality of Lemma \ref{lem:MainLemmaCorrected} by $H_P(Y\mid X)$, the inequality becomes
\begin{equation}
\mathbb{E}_{X,Y\sim P}|\bm{e}_X| + \mathbb{E}_{X,Y\sim P} [-\log Q(Y \mid X)] \ge H_P(Y \mid X) + c.
\end{equation}
Gibbs' inequality by itself gives
\begin{equation}
\mathbb{E}_{X,Y\sim P}[-\log Q(Y\mid X)]\ge H_P(Y\mid X),
\end{equation}
so the above inequality gives no implication about the explanation length $\mathbb{E}_{X,Y\sim P}|\bm{e}_X|$.
\end{remark}
The above remark suggests that the significance of Lemma \ref{lem:MainLemmaCorrected} lies in using the expectation of the algorithmic complexity of individual outputs rather than entropy.
The following example shows that the gap between them can be arbitrarily large.
\begin{proposition}[Example in which the lower bound given by the Quadrilemma is larger than the entropy lower bound]
\label{prp:KCEntropyGap}
For every $n\in\NN$, there exists a probability distribution $P$ such that
\begin{equation}
\EE_{X,Y\sim P}K_V(Y\mid X)\ge H_P(Y\mid X)+n
\end{equation}
holds.
\end{proposition}

\begin{proof}
Consider the set of all sufficiently long strings.
There are $|\Sigma|^m$ strings over $\Sigma$ of length $m$, while the number of strings satisfying $K_V(y)<N$ is at most
\begin{equation}
\sum_{r=0}^{N-1}|\Sigma|^r=\frac{|\Sigma|^N-1}{|\Sigma|-1}.
\end{equation}
If $m$ is chosen sufficiently large, then there exists a string of length $m$ satisfying $K_V(y_0)\ge N$.
The value of $N$ will be chosen later.

Define $P$ as the point-mass distribution
\begin{equation}
P(X=(),Y=y_0)=1.
\end{equation}
Then $Y$ is deterministic given $X$, and therefore
\begin{equation}
H_P(Y\mid X)=0.
\end{equation}
On the other hand,
\begin{equation}
\EE_{X,Y\sim P}K_V(Y\mid X)=K_V(y_0\mid ()).
\end{equation}
The difference caused by conditioning on the empty string is only a fixed constant.
That is, there exists a constant $c_0$ such that
\begin{equation}
K_V(y_0\mid ())\ge K_V(y_0)-c_0
\end{equation}
holds.
Taking $N:=n+c_0$ and choosing $y_0$ satisfying $K_V(y_0)\ge N$, we obtain
\begin{equation}
\EE_{X,Y\sim P}K_V(Y\mid X)\ge n=H_P(Y\mid X)+n.
\end{equation}
\end{proof}

\section{Limitation of the Naive Method}
\label{sec:NaiveResult}

In this section, we show the most naive lower bound on the expected conditional perplexity obtained from Kolmogorov complexity, and explain that this method cannot show a clear dependence between the length of an explanation of stochastic input-output AI and conditional perplexity.

The Kolmogorov complexity of an object represents the length of the shortest string that is transformed into that object by computation.
Therefore, it can be seen relatively easily from the definition that a stochastic input-output AI that can be described by a string shorter than the Kolmogorov complexity of the true input-output distribution cannot achieve the true distribution, and hence cannot achieve the minimum value of conditional perplexity.
To state this precisely in the general case requires detailed definitions, but when the support of the true distribution on the input space is the entire input space, the following concretely holds.

\begin{proposition}[Impossibility of minimizing conditional perplexity based on the Kolmogorov complexity of the true input-output distribution]
Let the input space be $\mathcal{X}$ and the output space be $\mathcal{Y}$.
Let $\pi\in\mathcal{P}(\mathcal{X})$ be the true probability distribution on the input space, and assume $\supp\pi=\mathcal{X}$.
There exists a constant $c_{\Interpret} \in \NN$ depending only on the universal conditional function $U$ and the interpretation function such that the following holds.
Let
\[
P(\cdot\mid\cdot), Q(\cdot\mid\cdot):\mathcal X\times\mathcal Y\computableto[0,1]
\]
be computable conditional probability mass functions.
Assume that $\bm e\in\Lambda^*$ is a completely faithful explanation of $Q$ and that
\begin{equation}
L_\Lambda|\bm e|
<
C_U(P) - c_{\Interpret}
\end{equation}
holds.
Then $Q$ cannot minimize the expected logarithmic conditional perplexity.
More specifically,
\begin{equation}
\EE_{X\sim\pi,\ Y\sim P(\cdot\mid X)}[-\log Q(Y\mid X)]
>
\min_{P'(\cdot\mid\cdot)\in\mathcal P(\mathcal Y\mid\mathcal X)}
\EE_{X\sim\pi,\ Y\sim P(\cdot\mid X)}[-\log P'(Y\mid X)]
\end{equation}
holds.
\end{proposition}

\begin{proof}
First, we show from the assumption that $Q\ne P$.
If $Q=P$, then $\bm e$ is also a completely faithful explanation of $P$.
There exists a partial computable function $\tilde{\Interpret}: \subseteq \Sigma^* \computableto \Sigma^*$ such that $\tilde{\Interpret} \Big(\overline{\Enc_{\mathcal{X}} (x)} \cdot \overline{\Enc_{\mathcal{Y}} (y)} \cdot \overline{\Enc_{\mathcal{\NN}} (k)} \cdot \Enc_{\Lambda^*} (\bm{e}) \Big) \approx_k P(y \mid x)$.

By the universality of $U$, there exists $\bm{p}_{\tilde{\Interpret}} \in \Sigma^*$ such that
\begin{equation}
\begin{split}
& U \Big(\overline{\Enc_{\mathcal{X}} (x)} \cdot \overline{\Enc_{\mathcal{Y}} (y)} \cdot \overline{\Enc_{\mathcal{\NN}} (k)} \cdot \bm{p}_{\tilde{\Interpret}} \cdot \Enc_{\Lambda^*} (\bm{e}) \Big) \\
& = \tilde{\Interpret} \Big(\overline{\Enc_{\mathcal{X}} (x)} \cdot \overline{\Enc_{\mathcal{Y}} (y)} \cdot \overline{\Enc_{\mathcal{\NN}} (k)} \cdot \Enc_{\Lambda^*} (\bm{e}) \Big) \approx_k P(y \mid x) 
\end{split}
\end{equation}
holds.
Therefore, by the definition of the plain Kolmogorov complexity of real-valued functions,
\begin{equation}
C_U (P) \le \Big|\bm{p}_{\tilde{\Interpret}} \cdot \Enc_{\Lambda^*} (\bm{e})\Big| = |\bm{p}_{\tilde{\Interpret}}| + \Big|\Enc_{\Lambda^*} (\bm{e})\Big| \le |\bm{p}_{\tilde{\Interpret}}| + L_{\Lambda} |\bm{e}|.
\end{equation}
Taking $c_{\Interpret} = |\bm{p}_{\tilde{\Interpret}}|$, $c_{\Interpret}$ does not depend on $P,Q$, or $\bm{e}$, and this contradicts the assumption $L_\Lambda |\bm{e}| < C_U (P) - c_{\Interpret}$.
Therefore, in this case, $Q \ne P$.

Next, we discuss minimization of the expected log loss.
For any conditional probability mass function $P'$,
\begin{align}
&\EE_{X\sim\pi,\ Y\sim P(\cdot\mid X)}[-\log P'(Y\mid X)]\\
&=
H_P(Y\mid X)
+
\EE_{X\sim\pi}
D_{\mathrm{KL}}\big(P(\cdot\mid X)\,\|\,P'(\cdot\mid X)\big).
\end{align}
Here, the base of the logarithm is $|\Sigma|$.
By Gibbs' inequality, the KL divergence is nonnegative, and equality holds if and only if
\begin{equation}
P'(\cdot\mid X)=P(\cdot\mid X)
\end{equation}
holds $\pi$-almost surely.
Since $\supp\pi=\mathcal X$, this is equivalent to $P'(\cdot\mid x)=P(\cdot\mid x)$ for every $x\in\mathcal X$.

Since $Q\ne P$, there exists $x$ satisfying $Q(\cdot\mid x)\ne P(\cdot\mid x)$.
Because $\pi(x)>0$,
\begin{equation}
\EE_{X\sim\pi}
D_{\mathrm{KL}}\big(P(\cdot\mid X)\,\|\,Q(\cdot\mid X)\big)
>
0
\end{equation}
holds.
If there exists $y$ satisfying $P(y\mid x)>0$ and $Q(y\mid x)=0$, then the left-hand side is $+\infty$, and is again positive.
Therefore,
\begin{equation}
\EE_{X\sim\pi,\ Y\sim P(\cdot\mid X)}[-\log Q(Y\mid X)]
>
H_P(Y\mid X)
\end{equation}
holds.
On the other hand, by taking $P'=P$, the expected log loss becomes $H_P(Y\mid X)$.
Hence, $Q$ cannot minimize conditional perplexity.
\end{proof}

The above proposition means that, if one wants to minimize conditional perplexity exactly, there is a lower bound on the length of a complete and faithful explanation determined by the plain Kolmogorov complexity $C_U(P)$ of the true probability distribution $P$.
However, \Emph{the above proposition says nothing about cases in which conditional perplexity is not exactly minimal}.
In fact, there are examples in which $L_{\Lambda}|\bm e|$ is much shorter than $C_U(P)$, and nevertheless the optimization gap in expected logarithmic perplexity
\[
\EE_{X \sim \pi,\,Y|X \sim P}[-\log Q(Y\mid X)] - \min_{Q^* \in \mathcal{P} (\mathcal Y\mid\mathcal X)}\EE_{X \sim \pi,\,Y|X \sim P} \big[-\log Q^* (Y\mid X) \big]
\]
is extremely small.
More specifically, the following proposition holds.

\begin{proposition}[Failure of the naive lower bound based on $C_U(P)$]
\label{cor:FaithfulExplanationShorterThanTrueDistributionComplexity}
Let the input space, or source set, $\mathcal{X}$ and the output space, or target set, $\mathcal{Y}$ both be countably infinite sets.

For every $n\in\NN$ and every $\epsilon\in\RR_{>0}$, there exist a probability distribution $\pi\in\mathcal P(\mathcal X)$ whose support is all of $\mathcal X$, two computable conditional probability mass functions
\begin{equation}
P(\cdot\mid\cdot), Q(\cdot\mid\cdot):\mathcal X\times\mathcal Y\computableto[0,1],
\end{equation}
and a string $\bm{e} \in \Lambda^*$ that is a completely faithful explanation of $Q$ such that the following two conditions hold simultaneously:
\begin{itemize}
\item $\EE_{X \sim \pi,\,Y|X \sim P}[-\log Q(Y\mid X)] - \min_{Q^* \in \mathcal{P} (\mathcal Y\mid\mathcal X)}\EE_{X \sim \pi,\,Y|X \sim P} \big[-\log Q^* (Y\mid X) \big] < \epsilon$.
\item $|\bm{e}| < C_U (P)-n$.
\end{itemize}
\end{proposition}

\begin{proof}
First, we confirm that the optimization gap in expected logarithmic perplexity is given by the Kullback--Leibler divergence.
More specifically, for any conditional probability mass function $P'(\cdot\mid\cdot)$, the following decomposition holds:
\begin{equation}
\label{eqn:PerplexityHKL}
\begin{aligned}
&\EE_{X\sim\pi,\,Y\sim P(\cdot\mid X)}
[-\log Q(Y\mid X)]
\\
&\qquad
=
H(Y \mid X)
+
\EE_{X\sim\pi}
D_{\mathrm{KL}}\big(P(\bullet \mid X) \ \big\|\ Q (\bullet \mid X)\big).
\end{aligned}
\end{equation}
Here, $H (Y \mid X) := \EE_{X\sim\pi,\,Y\sim P(\cdot\mid X)}
[-\log P (Y\mid X)]$ is the conditional entropy of $X$ and $Y$ with $X\sim\pi,\,Y\sim P(\cdot\mid X)$, and $D_{\mathrm{KL}}$ is the Kullback--Leibler divergence function.
By Gibbs' inequality, the second term in Equation \eqref{eqn:PerplexityHKL} is nonnegative, and the condition $Q (\bullet \mid x) = P (\bullet \mid x)$
for all $x \in \mathcal{X}$ is a necessary and sufficient condition for the second term to attain its minimum value zero.
Therefore, the optimization gap in expected logarithmic perplexity is represented as follows:
\begin{equation}
\begin{split}
&\EE_{X \sim \pi,\,Y|X \sim P}[-\log Q(Y\mid X)]
-
\min_{Q^* \in \mathcal{P} (\mathcal Y\mid\mathcal X)}\EE_{X \sim \pi,\,Y|X \sim P} \big[-\log Q^* (Y\mid X) \big]
\\
&=
\EE_{X\sim\pi}
D_{\mathrm{KL}}\big(P (\bullet \mid X) \ \big\|\ Q (\bullet \mid X)\big).
\end{split}
\end{equation}
Thus, the conditions to be satisfied are
\begin{itemize}
\item $\EE_{X\sim\pi}
D_{\mathrm{KL}}\big(P (\bullet \mid X) \ \big\|\ Q (\bullet \mid X)\big) < \epsilon$.
\item $|\bm{e}| < C_U (P)-n$.
\end{itemize}

In what follows, we construct a concrete example satisfying the above two conditions.
As a full-support probability distribution $\pi$ on the input space, fix $\pi (x) = \frac{1}{(i(x)+1)(i(x)+2)}$; in fact, the concrete values of $\pi$ are irrelevant to the subsequent argument.
Here, $i: \mathcal{X} \to \NN$ is the lexicographic index of the code of each element of $\mathcal{X}$ under $\Enc_{\mathcal{X}}$.
Also fix a computable full-support probability distribution on $\mathcal{Y}$ by $q(y) := \frac{1}{(j(y)+1)(j(y)+2)}$, and set $Q (y \mid x) = q(y)$. Here, $j: \mathcal{Y} \to \NN$ is the lexicographic index of the code of each element of $\mathcal{Y}$ under $\Enc_{\mathcal{Y}}$.

The idea of the proof is as follows.
As the true conditional probability mass function $P (\cdot \mid \cdot)$, consider one whose probability mass differs slightly from $q$ on two distinct elements $\eta, \eta' \in \mathcal{Y}$ having large Kolmogorov complexity.
More specifically, define it as follows.
Let $P (y \mid x) = p (y)$, where
\begin{equation}
p (y) =
\begin{cases}
q (y) & \ \text{if $y \ne \eta, y \ne \eta'$,} \\
q (y) + \delta & \ \text{if $y = \eta$,} \\
q (y) - \delta & \ \text{if $y = \eta'$,} \\
\end{cases}
\end{equation}
where $\eta, \eta' \in \mathcal{Y}$ and $\delta \in \big(0, q(\eta')\big)$ will be chosen appropriately later.
Since $P$ and $Q (\cdot \mid \cdot)$ differ only slightly as probability distributions, the optimization gap of $Q$ is small. However, by detecting the difference between $Q$ and $P$, one can describe $\eta, \eta' \in \mathcal{Y}$, and hence it follows that the Kolmogorov complexity of $P$ is large.
In fact, by choosing $\eta, \eta'$ appropriately, it follows that the Kolmogorov complexity of $P$ can be made arbitrarily large.

We now define $\eta, \eta', \delta$ concretely.
First, in order to quantify the relation between $C_U (P)$ and $\eta$, consider the following algorithm $\mathsf{FindDiff}$, which outputs $\bm{\eta} = \Enc_{\mathcal{Y}} (\eta)$ using $\bm{p} \in \mathsf{Program}_{U} (P)$.
\begin{itemize}
\item \textbf{Input}: $\bm{p} \in \mathsf{Program}_{U} (P)$.
\item \textbf{Step 1}: Initialize $j \gets 0$. Choose an arbitrary $x_0 \in \mathcal{X}$, and let $\bm{x}_0 := \Enc_{\mathcal{X}} (x_0)$.
\item \textbf{Step 2}: Find $\bm{y}_j$, the $j$-th element in the lexicographic ordering of the code set $\Enc_{\mathcal{Y}} (\mathcal{Y})$. Define $y_j \in \mathcal{Y}$ as the element satisfying $\Enc_{\mathcal{Y}} (y_j) = \bm{y}_j$.
\item \textbf{Step 3}: Using $\bm{p}$, $\bm{x}_0$, and $\bm{y}_j$, compute an approximation $\tilde{p}_j$ of $P (y_j \mid x_0)$ with accuracy $\pm \frac{\delta}{2}$.
\item \textbf{Step 4}:
Compute $q_j = Q (y_j \mid x_0) = \frac{1}{(j+1)(j+2)}$. Note that the right-hand side is a rational number, so its exact value can be computed.
\item \textbf{Step 5}: If $\tilde{p}_j > q_j$, \textbf{go to Step 7}.
\item \textbf{Step 6}: Set $j \gets j + 1$, and \textbf{go to Step 2}.
\item \textbf{Step 7}: \textbf{Output} $\bm{\eta} = \bm{y}_j$ \textbf{then terminate}.
\end{itemize}

By the universality of $U$, there exists a string $\bm{p}_\mathsf{FindDiff} \in \Sigma^*$ such that $\bm{\eta} = U (\overline{\bm{p}_\mathsf{FindDiff}} \cdot \bm{p})$.
This gives $C_U (\eta) \le |\bm{p}_\mathsf{FindDiff}| + |\bm{p}|$.
Applying this to $\bm{p}^* \in \mathsf{Program}_U (P)$ satisfying $|\bm{p}^*| = C_U (P)$, we obtain $C_U (\eta) \le |\bm{p}_\mathsf{FindDiff}| + C_U (P)$.
That is,
\begin{equation}
\label{eqn:CPEval}
C_U (P) \ge C_U (\eta) - |\bm{p}_\mathsf{FindDiff}|.
\end{equation}

On the other hand, evaluating the relation between the optimization gap in expected perplexity and $\delta, \eta, \eta'$ gives the following:
\begin{equation}
\label{eqn:KLEval}
\begin{split}
& \EE_{X\sim\pi}
D_{\mathrm{KL}}\big(P (\bullet \mid X) \ \big\|\ Q (\bullet \mid X)\big) \\
& \le \EE_{X\sim\pi}
D_{\mathrm{KL}} (p \ \big\|\ q) \\
& = D_{\mathrm{KL}} (p \ \big\|\ q) \\
& = (q(\eta) + \delta) \log_{|\Sigma|} \frac{q(\eta) + \delta}{q(\eta)} + (q(\eta') - \delta) \log_{|\Sigma|} \frac{q(\eta') - \delta}{q(\eta')} \\
& \le \frac{1}{\ln |\Sigma|} \cdot \Bigg[(q(\eta) + \delta) \bigg(\frac{q(\eta) + \delta}{q(\eta)} - 1\bigg) + (q(\eta') - \delta) \bigg( \frac{q(\eta') - \delta}{q(\eta')} - 1\bigg) \Bigg] \\
& = \frac{\delta^2}{\ln |\Sigma|} \left( \frac{1}{q (\eta)} + \frac{1}{q (\eta')}) \right).
\end{split}
\end{equation}
In the second inequality above, we used $\log_{|\Sigma|} (t) = \frac{\ln t}{\ln |\Sigma|} \le \frac{t - 1}{\ln |\Sigma|}$.

It remains only to define $\eta$, $\eta'$, and $\delta$ appropriately.
First, choose $\eta$ so that $C_U (\eta) > |\bm{p}_\mathsf{FindDiff}| + |\bm{e}| + n$.
Such an $\eta$ necessarily exists.
Indeed, the number of $y \in \mathcal{Y}$ satisfying $C_U (y) \le |\bm{p}_\mathsf{FindDiff}| + |\bm{e}| + n$ is at most $\sum_{\ell=0}^{|\bm{p}_\mathsf{FindDiff}| + |\bm{e}| + n} |\Sigma|^\ell$, and is in particular finite, whereas $\mathcal{Y}$ is countably infinite.
When $\eta$ is chosen as above, Equation \eqref{eqn:KLEval} implies that $C_U (P) > |\bm{e}| + n$ regardless of the concrete choice of $\eta'$ or $\delta \in (0, q(\eta'))$.
Also let $\eta' = y_0$, where $y_0$ is the element of $\mathcal{Y}$ satisfying $j(y_0) = 0$.
As $\delta > 0$, choose a rational number satisfying
$\delta < \min \Big\{\sqrt{\frac{q (\eta) q (\eta')}{q (\eta) + q (\eta')} \ln |\Sigma| \cdot \epsilon}, q (\eta') \Big\}$.
The $P$ defined by the above choices is a full-support conditional probability mass function, and Equation \eqref{eqn:KLEval} gives $\EE_{X\sim\pi}
D_{\mathrm{KL}}\big(P (\bullet \mid X) \ \big\|\ Q (\bullet \mid X)\big) < \epsilon$.
This completes the proof.
\end{proof}
For the above reason, an analysis using the Kolmogorov complexity $C_U(P)$ of the true probability mass function cannot provide conditions on the explanation length of a stochastic input-output AI whose conditional perplexity is not minimal but is extremely close to minimal.

\end{document}